\documentclass[twoside,11pt]{article}

\makeatletter
\newcommand{\NAT@parse}[1]{}
\newcommand{\NAT@force@numbers}{}
\newcommand{\bibpunct}[6]{}
\@namedef{ver@natbib.sty}{2999/99/99} %
\makeatother

\let\originalbibliographystyle\bibliographystyle
\renewcommand{\bibliographystyle}[1]{}
\usepackage{jmlr2e}
\let\bibliographystyle\originalbibliographystyle
\usepackage{amsmath}
\usepackage{amssymb}
\usepackage{bm}
\usepackage{booktabs}
\usepackage[font=footnotesize,labelfont=bf]{caption}
\usepackage[capitalize,nameinlink]{cleveref}
\usepackage{color}
\usepackage{dsfont}
\usepackage[T1]{fontenc}
\usepackage{mathabx}
\usepackage{mathtools}
\usepackage{mathrsfs}
\usepackage{microtype}
\usepackage{multirow}
\usepackage{pifont}
\usepackage{rotating}
\usepackage[group-separator={,},group-minimum-digits={3}]{siunitx}
\usepackage{subcaption}
\usepackage{tikz}
\usepackage{thm-restate}
\usepackage[normalem]{ulem}
\usepackage{wrapfig}

\usepackage{pgfplots}
\pgfplotsset{compat=1.18}

\hypersetup{
  hidelinks
}

\crefformat{section}{\S#2#1#3}
\crefname{definition}{Definition}{Definition}
\crefname{proposition}{Proposition}{Propositions}
\crefname{lemma}{Lemma}{Lemmas}
\crefrangeformat{theorem}{Theorems~#3#1#4{--}#5#2#6}
\crefrangeformat{proposition}{Propositions~#3#1#4{--}#5#2#6}
\crefrangeformat{lemma}{Lemmas~#3#1#4{--}#5#2#6}

\renewcommand{\tilde}{\widetilde}
\renewcommand{\hat}{\widehat}
\renewcommand{\check}{\widecheck}

\newcommand{\Bcal}{\mathcal{B}}

\newcommand{\Dcal}{\mathcal{D}}

\newcommand{\Hcal}{\mathcal{H}}
\newcommand{\Ical}{\mathcal{I}}
\newcommand{\Jcal}{\mathcal{J}}
\newcommand{\Kcal}{\mathcal{K}}

\newcommand{\Mcal}{\mathcal{M}}

\newcommand{\Ucal}{\mathcal{U}}

\newcommand{\Xcal}{\mathcal{X}}
\newcommand{\Ycal}{\mathcal{Y}}

\newcommand{\Lbb}{\mathbb{L}}

\newcommand{\Nbb}{\mathbb{N}}

\newcommand{\Rbb}{\mathbb{R}}

\newcommand{\Zbb}{\mathbb{Z}}

\newcommand{\Cbf}{\mathbf{C}}

\newcommand{\Ibf}{\mathbf{I}}

\newcommand{\Lbf}{\mathbf{L}}

\newcommand{\Obf}{\mathbf{O}}

\newcommand{\ebf}{\mathbf{e}}

\newcommand{\qbf}{\mathbf{q}}

\newcommand{\ubf}{\mathbf{u}}
\newcommand{\vbf}{\mathbf{v}}
\newcommand{\wbf}{\mathbf{w}}
\newcommand{\xbf}{\mathbf{x}}

\newcommand{\zbf}{\mathbf{z}}

\newcommand{\etabf}{\bm{\eta}}

\newcommand{\lambdabf}{\bm{\lambda}}

\newcommand{\xibf}{\bm{\xi}}

\newcommand{\onebf}{\mathbf{1}}
\newcommand{\ellbf}{\bm{\ell}}

\newcommand{\defeq}{\coloneqq}
\newcommand{\eqdef}{\eqqcolon}
\renewcommand{\subset}{\subseteq}

\DeclareMathOperator{\domain}{\mathrm{dom}}

\DeclareMathOperator*{\argmax}{arg\,max}
\DeclareMathOperator*{\argmin}{arg\,min}

\DeclareMathOperator{\supp}{\mathrm{supp}}

\newcommand{\indicator}[1]{\mathds{1}_{\{#1\}}}
\newcommand{\set}[1]{\left\lbrace{#1}\right\rbrace}
\newcommand{\setcomp}[2]{\left\lbrace{#1} \relmiddle| {#2}\right\rbrace}
\newcommand{\inpr}[2]{\left\langle{#1},{#2}\right\rangle}
\newcommand{\relmiddle}[1]{\mathrel{}\middle#1\mathrel{}}

\newcommand{\pdiff}[2]{\frac{\partial{#1}}{\partial{#2}}}

\newcommand{\Hess}{\mathop{\mathrm{Hess}}}
\newcommand{\rd}{\mathrm{d}}

\renewcommand{\epsilon}{\varepsilon}

\newcommand{\simplex}{\triangle^N}
\newcommand{\vecarr}[2]{{\begin{bsmallmatrix*}{#1}\\{#2}\end{bsmallmatrix*}}}

\newcommand{\openbox}{\leavevmode
  \hbox to.77778em{%
  \hfil\vrule
  \vbox to.675em{\hrule width.6em\vfil\hrule}%
  \vrule\hfil}}
\renewenvironment{proof}{\par\noindent{\bf Proof\ }}{\hfill$\openbox$}

\usepackage{lastpage}
\jmlrheading{26}{2025}{1-\pageref{LastPage}}{7/24; Revised 10/25}{12/25}{24-1112}{Han Bao and Asuka Takatsu}

\ShortHeadings{Proper losses regret at least $1/2$-order}{Bao and Takatsu}
\firstpageno{1}

\begin{document}

\title{Proper losses regret at least $1/2$-order}

\author{\name Han Bao \email bao.han@ism.ac.jp  \\
       \addr The Institute of Statistical Mathematics
       \AND
       \name Asuka Takatsu \email asuka-takatsu@g.ecc.u-tokyo.ac.jp \\
       \addr The University of Tokyo
       }

\editor{Mehryar Mohri}

\maketitle

\begin{abstract}
  A fundamental challenge in machine learning is the choice of a loss as it characterizes our learning task, is minimized in the training phase, and serves as an evaluation criterion for estimators.
  Proper losses are commonly chosen, ensuring minimizers of the full risk match the true probability vector.
  Estimators induced from a proper loss are widely used to construct forecasters for downstream tasks such as classification and ranking.
  In this procedure, how does the forecaster based on the obtained estimator perform well under a given downstream task?
  This question is substantially relevant to the behavior of the $p$-norm between true probability and estimated vectors when the estimator is updated.
  In the proper loss framework, the suboptimality of the estimated probability vector from the true probability vector is measured by a surrogate regret.
  First, we analyze a surrogate regret and show that the \emph{strict} properness of a loss is necessary and sufficient to establish a non-vacuous surrogate regret bound.
  Second, we tackle an important open question that the order of convergence in $p$-norm cannot be faster than the $1/2$-order of surrogate regrets for a broad class of strictly proper losses.
  This implies that strongly proper losses asymptotically achieve the optimal convergence rate.
\end{abstract}


\section{Introduction}
\label{section:introduction}
\emph{Proper losses}, also known as proper scoring rules, are measurements of the quality of a probabilistic prediction given a true probability vector \cite{Buja:2005,Gneiting:2007,Reid:2010}.
Intuitively, we say a loss is \emph{proper} if the target probability vector is its minimizer, and \emph{strictly proper} if the minimizer is unique, which are a basic property for a reasonable loss.
Proper losses are prevailing in modern machine learning: for example, the cross-entropy loss popular in deep learning essentially corresponds to the log loss (or logarithmic score), and the Brier score is used for assessing model uncertainties \cite{Ovadia:2019} \cite{Gruber2022NeurIPS}.
As such, probabilistic estimators are obtained via proper loss minimization.
It is common to post-process a minimizer of a proper loss for downstream tasks,
such as classification (by choosing the most likely label), ranking (by giving ranking scores to each label \cite{Narasimhan:2013}), F-measure optimization (by thresholding the estimated probability \cite{Koyejo:2014}), and probability calibration \cite{Kull:2017,Blasiok:2023}.
Here, we are interested in the predictive performance of post-processed estimators in downstream tasks.
Given the true and estimated probability vectors $\qbf$ and $\hat\qbf$, respectively, the \emph{surrogate regret} $R(\qbf, \hat\qbf)$ (introduced in \cref{section:proper_losses}) measures the suboptimality of $\hat\qbf$ from $\qbf$ in terms of a proper loss.
Can we relate the suboptimality of a forecaster for a downstream task to the surrogate regret?

Surrogate regret bounds relate the surrogate regret to the performance for downstream tasks, and have been derived for binary classification \cite{Zhang:2004,Reid:2009:ICML}, bipartite ranking \cite{Kotlowski2011ICML,Agarwal:2014}, property elicitation \cite{Agarwal:2015}, F-measure optimization \cite{Kotlowski:2016,Zhang:2020}, and learning with noisy labels \cite{Natarajan2013NeurIPS,Zhang:2021}, independently.
Recently, a unified surrogate regret bound across different downstream tasks has been established \cite{Bao:2023}, where surrogate regret bounds are unified in terms of the $1$-norm.
This is based on the observation that the suboptimality of the aforementioned downstream tasks can be controlled by the $1$-norm.
However, the derived bound has been limited to the binary classification case, and it remains unclear when the surrogate regret bound is non-vacuous.
A reasonable loss should entail a non-vacuous regret bound, which is crucial to tackling numerous downstream tasks simultaneously.
Moreover, an important conjecture that the convergence rate of surrogate regret bounds cannot be faster than the $1/2$-order has yet to be solved.
This conjecture has a significant role in the choice of losses because the lower bound of the order of convergence contributes to delineating the optimality of a given proper loss.

In this article, we aim to study when the surrogate regret bounds are non-vacuous and how fast the order of convergence in terms of the~$p$-norm.
To this end, we analyze the $p$-norm bounds by the surrogate regret $R(\qbf, \hat\qbf)$ jointly with a rate function $\psi$ in the following form by extending from the binary classification case \cite{Bao:2023} to the multiclass classification case:
\begin{equation}
  \label{equation:desiderata}
  \|\qbf - \hat\qbf\|_p \le \psi(R(\qbf, \hat\qbf)).
\end{equation}
After formalizing these notions in \cref{section:proper_losses}, we derive the surrogate regret bounds in \cref{section:regret_bounds}.
To derive bounds of the form \eqref{equation:desiderata}, we introduce the \emph{moduli of convexity} \cite{Polyak1966} \cite{Figiel:1976}, which describe the information of its second derivative of convex functions (on the probability simplex).
The rate~$\psi$ in \eqref{equation:desiderata} can be characterized by the modulus of a Bregman generator function associated with a proper loss~$\ellbf$ (\cref{theorem:regret_lower_bound}).
We first show that the strict properness of a loss is necessary and sufficient to obtain a non-vacuous surrogate regret bound, or strictly increasing~$\psi$, to put it differently (\cref{proposition:moduli_monotonicity}).
Whereas it has been known that non-strictly proper losses can achieve non-vacuous bounds for classification \cite[Corollary 27]{Reid:2011},
our sufficiency result argues that the strict properness is a minimal requirement for an estimate to be non-vacuous in terms of the~$p$-norm.
As our second main result, we provide an affirmative answer to the above conjecture: the optimal rate~$\psi(\rho)$ as~$\rho \downarrow 0$ is~$O(\rho^{1/2})$, for a broad class of proper losses (\cref{theorem:regret_order}).
This convergence rate has already been known for a restricted class of proper losses, known as strongly proper losses \cite{Agarwal:2014}.
Hence, our result ensures the asymptotic optimality of strongly proper losses.
This gives an answer to the question, ``Do we have an interesting loss that is strictly proper but not strongly proper?'' \cite[\S6.2.9]{Bao2022Thesis}: there is no better proper loss outside of strongly proper losses, as long as we are concerned with the asymptotic rate of~$\psi$ in \cref{equation:desiderata}.

\subsection{Organization}
The organization of this article and our contributions are summarized as follows.
\begin{itemize}
  \item \Cref{section:background}: Notation and necessary backgrounds on convex analysis are summarized.
  We tailor subdifferentials for convex functions defined on the probability simplex.

  \item \Cref{section:proper_losses}: Proper losses for multiclass classification are introduced. \Cref{lemma:minimizability} characterizes the existence of minimizers for general losses, and \cref{proposition:savage} gives a self-contained and rigorous proof of the well-known representation of proper losses \cite{Savage:1971}.

  \item \Cref{section:regret_bounds}: 
  \Cref{proposition:moduli_monotonicity} is our first result, proving that the strict properness of a loss is a necessary and sufficient condition for an associated surrogate regret bound to be non-vacuous.
  \Cref{theorem:regret_lower_bound} extends surrogate regret bounds for binary classification \cite{Bao:2023} to multiclass classification. This is achieved by extending the moduli of convexity to multivariate functions (\cref{definition:modulus}).

  Then, the benefits of \cref{theorem:regret_lower_bound} are discussed in \cref{section:downstream}.
  In particular, we can obtain the~$p$-norm bound in the form of~\eqref{equation:desiderata}, which can be used to control the performance of plug-in forecasters for downstream tasks such as multiclass classification, learning with noisy labels, and bipartite ranking.

  \item \Cref{section:regret_order}: We evaluate the rate~$\psi(\rho)$ by power functions such as~$\rho^{1/s} \lesssim \psi(\rho) \lesssim \rho^{1/S}$ for some constants~$s, S > 0$,\footnote{In our notation, $\psi_1 \lesssim \psi_2$ indicates the existence of an absolute constant $C > 0$ such that $C\psi_1 \leq \psi_2$.} which is based on the Simonenko order function previously adopted \cite{Bao:2023}. Our second main result roughly shows that~$s \geq 2$ (\cref{theorem:regret_order}), establishing the asymptotic optimality~$\psi(\rho) \gtrsim \rho^{1/2}$ of strongly proper losses.

  \item \Cref{section:examples}: Several examples of convex functions to generate proper losses are discussed.
\end{itemize}

\section{Background}
\label{section:background}
In this section, we summarize the notation and basic properties of convex functions.

\subsection{Notation}
Throughout this article, fix~$N \in \Nbb$ and~$p \in [1, \infty]$.
The Kronecker delta is denoted by $\delta_{ij}$.
For~$k \in \Nbb$, we set~$[k] \defeq \set{1, 2, \dots, k}$.
A vector is denoted by bold-face such as~$\xibf \in \Rbb^N$, and its $n$-th (scalar) component is written as non-bold~$\xi_n$ for each~$n \in [N]$.
The~$p$-norm of~$\xibf \in \Rbb^N$ is denoted by~$\|\xibf\|_p$.
For a topology on~$\Rbb^N$, we refer to one induced from the~$2$-norm, but it makes no difference whichever norm we choose.
Similarly, a convexity of a function on~$(\Rbb^N,\|\cdot\|_p)$ is determined independently of the choice of~$p$.
The standard inner product on~$\Rbb^N$ is denoted by $\inpr{\xibf}{\xibf'} \defeq \sum_{n \in [N]} \xi_n\xi_n'$.
We introduce the notation 
\begin{align*}
  \simplex \defeq \setcomp{\qbf \in \Rbb^N}{q_n \geq 0\ (n \in [N]), \ \inpr{\qbf}{\onebf} = 1}, \quad
  \simplex_+ \defeq \setcomp{\qbf \in \simplex}{q_n> 0\ (n \in [N])},
\end{align*}
where~$\onebf \in \Rbb^N$ is the vector with each component being one.
For~$\qbf\in \simplex$, we denote by~$\supp(\qbf)$ the \emph{support} of~$\qbf$, that is,
\[
  \supp(\qbf) \defeq \setcomp{n\in[N]}{q_n>0}.
\]
We adhere to the convention that
\[
  \pm \infty \leq \pm\infty, \;\; a\pm\infty=\pm\infty, \;\; b\cdot (\pm\infty)=\pm\infty, \;\; -b\cdot(\pm\infty) = \mp\infty,\;\; 0\cdot (\pm\infty)=0,
  \;\; \ln 0=-\infty,
\]
for~$a \in \Rbb$ and~$b > 0$.
We use~$O$ and~$\Omega$ to denote the \emph{infinitesimal} asymptotic order.
To be precise, for two functions~$\phi,\psi$ defined around~$0$,~$\phi(\epsilon) = O(\psi(\epsilon))$ and~$\phi(\epsilon)=\Omega(\psi(\epsilon))$ as~$\epsilon \downarrow 0$ should be understood as 
\[
  \limsup_{\epsilon \downarrow 0} \left|\frac{\phi(\epsilon)}{\psi(\epsilon)}\right| < \infty
  \quad \text{and} \quad
  \liminf_{\epsilon \downarrow 0} \left|\frac{\phi(\epsilon)}{\psi(\epsilon)}\right| > 0,
\]
respectively.

\subsection{Convex analysis}
\label{section:convex_analysis}
In this subsection, let~$f:\Rbb^N\to(-\infty, \infty]$ denote a proper convex function on~$\Rbb^N$.
Its \emph{effective domain} is defined by 
\[
  \domain{f} \defeq \setcomp{\xibf \in \Rbb^N}{f(\xibf) < \infty}.
\]
For a convex set $S\subseteq\Rbb^N$, a convex function~$f:S\to(-\infty,\infty]$ is said \emph{strongly convex} on~$S$ with respect to the~$p$-norm if we have
\[
(1-t)f(\xibf)+tf\bigl(\check\xibf\bigr) - f\bigl((1-t)\xibf+t\check\xibf\bigr) \ge \frac12\kappa t(1-t)\bigl\|\xibf-\check\xibf\bigr\|_p^2
\quad \text{for $\xibf,\check\xibf\in S$, $t\in(0,1)$},
\]
for some~$\kappa>0$.
A vector~$\vbf \in \Rbb^N$ is called a \emph{subgradient} of~$f$ at~$\xibf^0 \in \Rbb^N$ if
\begin{equation}
  \label{equation:subtangent_inequality}
  f(\xibf) \ge f(\xibf^0) + \inpr{\smash{\vbf}}{\xibf - \smash{\xibf^0}}
  \quad \text{for all $\xibf \in \Rbb^N$.}
\end{equation}
We say that~$f$ is \emph{subdifferentiable} at~$\xibf^0$ if there exists a subgradient of~$f$ at~$\xibf^0$.

Assume~$\simplex \subset \domain{f}$.
Then $f$ is subdifferentiable at~$\qbf^0 \in \simplex_+$ \cite[Theorem~23.4]{Rockafeller:1970}. 
For~$\qbf^0 \in \simplex$, 
we extended the set of subgradient of $f$ at $\qbf^0$ as
\[
  \partial f(\qbf^0) =\setcomp{\vbf \in [-\infty,\infty)^N}{
    f(\qbf) \ge f(\qbf^0) + \inpr{\vbf}{\qbf-\qbf^0} \text{~holds for all $\qbf\in\simplex$}
}.
\]
The notion of~$\partial f(\qbf^0)$ differs from the usual one \cite[\S23]{Rockafeller:1970} due to the restriction of the domain of~$f$ from~$\Rbb^N$ to~$\simplex$
and the extension of the range of~$\vbf$ from~$ \Rbb^N$ to~$[-\infty,\infty)^N$.
This relaxation is useful to accommodate regular losses (given in \cref{definition:regular} later).
Note that, for $\vbf\in \partial f(\qbf^0)$,
$v_n=-\infty$ happens only for $n\notin \supp (\qbf^0)$.
We will show that 
$\partial f(\qbf^0)$ is nonempty for any $\qbf^0\in \simplex$ in \cref{section:subgradient_inequality}.

We call a map $\partial{f}: \simplex \to 2^{[-\infty,\infty)^N}$ the \emph{subdifferential} of~$f$.

\paragraph*{Bregman divergence.}
Let us denote an arbitrary \emph{selector} of~$\partial f(\qbf)$ by $\nabla f$, that is, a map assigning to each point~$\qbf \in \simplex$ an element in~$\partial f(\qbf)$.
Since~$\partial{f}(\qbf^0)$ consists of the gradient of~$f$ at~$\qbf^0$ if~$f$ is differentiable at~$\qbf^0$,
it is consistent to use the notation~$\nabla f$ for a selector.
For~$\qbf, \qbf^0 \in \simplex$, the associated \emph{Bregman divergence}~\cite{Bregman:1967} of~$\qbf$ given~$\qbf^0$ is defined by
\[
  B_{(f,\nabla f)}(\qbf\|\qbf^0) \defeq f(\qbf) - f(\qbf^0) - \inpr{\nabla f(\qbf^0)}{\qbf - \qbf^0} \in [0, \infty].
\]
Here,~$f$ is called the \emph{generator} of the Bregman divergence~$B_{(f,\nabla f)}$.

\paragraph*{Subgradient equivalence.}
Note that the definition of the subgradient introduced here restricts the domain to~$\simplex$.
Therefore, if~$\vbf \in \partial f(\qbf^0)$ for~$\qbf^0\in\simplex$,
then for~$\gamma:\simplex\to\Rbb$, we have
\begin{align*}
    f(\qbf)
    &\ge f(\qbf^0) + \inpr{\vbf}{\qbf - \qbf^0} \\
    &= f(\qbf^0) + \inpr{\vbf}{\qbf - \qbf^0} + \gamma(\qbf^0)\cdot \inpr{\onebf}{\qbf - \qbf^0} && \text{(because $\qbf,\qbf^0\in\simplex$)} \\
    &= f(\qbf^0) + \inpr{\vbf + \gamma(\qbf^0)\cdot\onebf}{\qbf - \qbf^0}.
  \end{align*}
Thus, 
for a selector $\nabla f$ of $\partial f$ and a function $\gamma$ on $\simplex$,
$\nabla f+\gamma\onebf$ is also a selector of $\partial f$
and 
\begin{equation}
  \label{equation:equivalence}
  \begin{aligned}
B_{(f,\nabla f)}= B_{(f, \nabla f+\gamma\onebf)}
  \end{aligned}
\end{equation}
holds. 
Readers must distinguish this notion of equivalence from the equivalence of scoring rules \cite[\S1.1]{Dawid2007AISM}.

\section{Classification, proper losses, and Savage representation}
\label{section:proper_losses}
After introducing the learning problem of multiclass classification, we discuss loss functions and their properties.
We review the notion of proper losses and its connection to Bregman divergences.
Although this connection is already known, we formalize it rigorously.
In particular, we verify the existence of minimizers of the conditional risk and its measurable selection (\cref{lemma:minimizability}), which have been implicitly used in previous literature without any proof.

\subsection{Multiclass classification}
We regard a Radon space~$\Xcal$ as an input space, that is, the set of possible observations, and~$\Ycal \defeq [N]$ as a set of labels.
The set~$\simplex$ is identified with the space of all probability measures on~$\Ycal$.
We fix a probability measure~$\nu$ on~$\Xcal \times \Ycal$ and denote by~$\nu_{\Xcal}$ the marginal of~$\nu$ on~$\Xcal$,
that is,~$\nu(\Bcal \times \Ycal) = \nu_{\Xcal}(\Bcal)$ holds for any measurable set~$\Bcal \subset \Xcal$.
Then, by the disintegration theorem \cite[Chapter III-70 and 72]{DellacherieMeyer78},
there exists a Borel map $\qbf: \Xcal \to \simplex$, uniquely defined~$\nu_{\Xcal}$-a.e.,
such that
\[
  \nu(\Bcal \times \set{y}) = \int_\Bcal q_y(\xbf) \rd\nu_{\Xcal}(\xbf)
\quad
\text{for a measurable set }\Bcal \subset \Xcal \text{ and }y \in \Ycal.
\]
We deem~$\qbf(\xbf) \in \simplex$ a true probability vector at the input~$\xbf$ induced from~$\nu$.
\emph{Multiclass classification} is a task to learn a forecaster $y: \Xcal \to \Ycal$ to predict the most likely label
\[
  y(\xbf) \in \argmax_{y \in \Ycal} q_y(\xbf)
  \quad \text{for each $\xbf \in \Xcal$.}
\]

\subsection{Proper losses}
\label{section:proper_loss_sub}
We continue to use the notation in the previous subsection.
To elicit the true probability map~$\qbf: \Xcal \to \simplex$, we use a \emph{loss}~$\ellbf$, which is a Borel map from~$\simplex$ to~$[0,\infty]^N$.
Define the associated \emph{full risk} by 
\[
  \Lbb[\hat{\qbf}] \defeq \int_{\Xcal \times \Ycal} \ell_y(\hat{\qbf}(\xbf)) \rd\nu(\xbf, y)
  \quad \text{for a Borel map~}
  \hat{\qbf}: \Xcal \to \simplex.
\]
A minimizer of~$\Lbb$ among Borel maps~$\hat{\qbf}: \Xcal \to \simplex$ is called an \emph{estimator} of~$\qbf: \Xcal \to \simplex$.
The choice of~$\ellbf$ directly affects the quality of an estimator.
It is more intuitive to work on the conditional counterpart of the full risk instead.
Let~$\qbf, \hat\qbf \in \simplex$ with slight abuse of notation.
For a loss~$\ellbf$, the associated \emph{conditional risk} of~$\hat\qbf$ given~$\qbf$ and \emph{conditional Bayes risk} of~$\qbf$ are defined by 
\begin{equation*}
  L(\qbf, \hat\qbf) \defeq \inpr{\qbf}{\ellbf(\hat{\qbf})} = \sum_{y \in \Ycal} q_y\ell_y(\hat\qbf) \quad \text{and} \quad
  \underline{L}(\qbf) \defeq \inf_{\hat\qbf \in \simplex} L(\qbf, \hat\qbf),
\end{equation*}
respectively.
Here, we regard~$\qbf\in\simplex$ as a \emph{true} probability vector
and~$\hat\qbf\in\simplex$ as an \emph{estimate}.
The full risk is rewritten as 
\[
  \Lbb[\hat{\qbf}] = \int_{\Xcal} L(\qbf(\xbf), \hat{\qbf}(\xbf)) \rd\nu_{\Xcal}(\xbf).
\]
Since the infimum of a family of linear functions is concave,~$\underline{L}$ is concave on~$\simplex$, consequently~$\underline{L} \circ \hat{\qbf}: \Xcal \to [-\infty,\infty)$ is measurable on~$\Xcal$, which in turn shows
\[
  \Lbb[\hat{\qbf}] \geq \int_{\Xcal} \underline{L}(\qbf(\xbf)) \rd\nu_{\Xcal}(\xbf).
\]
Thus, the minimization problem of the full risk is reduced to that of the conditional risk
if the map~$\Mcal: \simplex \to 2^{\simplex}$ defined by
\[
  \Mcal(\qbf) \defeq \argmin_{\hat\qbf \in \simplex} L(\qbf, \hat\qbf)
  \quad \text{for $\qbf \in \simplex$.} 
\]
has a Borel selector.
Note that~$\Mcal(\qbf) = \emptyset$ may happen, whose example is given in \cref{section:empty_minimizer_set}.

Next, we show that~$\Mcal$ has a Borel selector for continuous~$\ellbf$.
Although this fact is not directly relevant to our main topic, we detail it in this article because we are unaware of any previous literature formalizing it for losses defined on~$\simplex$.
For a different type of (margin-based) losses, the existence of a measurable full risk minimizer has been studied \cite[Theorem~3.2~(ii)]{Steinwart:2007}.
The proof is deferred to \cref{section:proofs}.
\begin{restatable}{lemma}{minimizability}
  \label{lemma:minimizability}
  Suppose~$\ellbf: \simplex \to [0,\infty]^N$ is lower semi-continuous.
  Then,~$\underline{L}(\qbf) \ge 0$ holds and~$\Mcal(\qbf)$ is nonempty and closed for~$\qbf \in \simplex$.
  Moreover, if~$\ellbf$ is continuous, then there exists a Borel selector of~$\Mcal$.
\end{restatable}
Our proof is based on a measurable selection theorem, which can be extended beyond proper losses to work on margin-based losses.
However, the proof crucially depends on the continuity of~$\ellbf$ to invoke the selection theorem.
This point is restrictive than the previous result \cite[Theorem~3.2~(ii)]{Steinwart:2007}, but we hope that the proof of \cref{lemma:minimizability} is more concise and provides an insight.

Since~$\Mcal(\qbf)$ is ideally a singleton of~$\qbf$, we consider such a class of losses.
\begin{definition}[{Proper losses~\cite{Winkler1968}}]
  A loss $\ellbf:\simplex \to [0,\infty]^N$ is said to be \emph{proper} if~$\qbf \in \Mcal(\qbf)$ holds for each~$\qbf \in \simplex$.
  We say~$\ellbf$ is \emph{strictly proper} if~$\Mcal(\qbf) = \set{\qbf}$ holds for each~$\qbf \in \simplex$.
\end{definition}
For a proper loss~$\ellbf$, the conditional risk is minimized at the true probability vector, and the identity map on~$\simplex$ becomes a Borel selector of~$\Mcal$.
In this case, it follows that~$\underline{L}(\qbf) = L(\qbf, \qbf)$ for~$\qbf \in \simplex$.
Note that recently discovered \emph{calm composite losses} are also valid loss functions on~$\simplex$ but a selector of~$\Mcal$ can be nonlinear, which generalize proper losses \cite{Bao2025AISTATS}.

\subsection{Savage representation}
\label{section:savage}
We will see that a proper loss has a connection with a Bregman divergence under the regularity.
\begin{definition}[{Regular losses \cite[Definition~1]{Gneiting:2007}}]
  \label{definition:regular}
  A loss~$\ellbf: \simplex \to [0, \infty]^N$ is said to be \emph{regular}
  if~$\ell_y(\qbf) = \infty$ happens only for~$y \notin \supp(\qbf)$.
\end{definition}
In what follows, we consider a regular loss $\ellbf:\simplex \to [0, \infty]^N$.

For a regular loss, we define the \emph{surrogate regret}~$R: \simplex \times \simplex \to [0, \infty]$ by
\begin{equation*}
  R(\qbf, \hat\qbf) \defeq L(\qbf, \hat\qbf) - \underline{L}(\qbf)
  \quad \text{for $\qbf, \hat\qbf \in \simplex$,}
\end{equation*}
which measures the suboptimality of an estimate~$\hat\qbf$ given a true~$\qbf$.
The surrogate regret will be used to assess the performance of~$\hat\qbf$ across different downstream tasks in \cref{section:downstream}.
We call~$R$ the \emph{surrogate} regret since it is a proxy performance measure to downstream tasks.

Although the following property
has been well known in literature \cite[\S4]{Savage:1971} \cite[Theorem~3.1]{Hendrickson1971} \cite[Theorem~2]{Gneiting:2007} \cite[Proposition~7]{Williamson:2016}, we provide its proof to handle the regularity and subdifferentials carefully.%
\footnote{
  None of the aforementioned previous literature dealt with subgradients whose elements possibly take~$-\infty$, though subtle.
  We carefully extend subdifferentials in \cref{section:convex_analysis} and check the subgradient inequality~\eqref{equation:subtangent_inequality} for such subgradients in the proof.
}
In essence, the negative Bayes risk~$-\underline{L}$ behaves as a Bregman generator therein.
\begin{restatable}[{Savage representation \cite[\S4]{Savage:1971}}]{proposition}{savage}
  \label{proposition:savage}
  Let~$\ellbf$ be regular.
  Then,~$\ellbf$ is proper (resp.\,strictly proper) if and only if
  there exists a proper convex (resp.\,strictly convex) function~$f$ on~$\Rbb^N$ such that~$\domain f = \simplex$ and, for all~$\hat\qbf \in \simplex$,
  there exists a subgradient~$\hat\vbf \in \partial{f}(\hat\qbf)$ satisfying 
  \begin{equation}
    \label{equation:savage}
    L(\qbf, \hat\qbf) = -f(\hat\qbf) - \inpr{\hat\vbf}{\qbf - \hat\qbf}
    \quad \text{for $\qbf \in \simplex$}.
  \end{equation} 
\end{restatable}
\begin{proof}
  First, assume that~$\ellbf$ is proper.
  Then,~$\underline{L}(\qbf) = L(\qbf, \qbf) \in \Rbb$. 
  Define~$f: \Rbb^N \to (-\infty, \infty]$ by 
  \begin{equation}
    \label{equation:generator}
    f(\xibf) \defeq \begin{cases} 
      -\underline{L}(\xibf) & \text{if } \xibf \in \simplex, \\
      \infty & \text{otherwise},
    \end{cases}
  \end{equation}
  then~$f$ is a proper convex function on~$\Rbb^N$ such that~$\domain f \subset \simplex$.
  For~$\qbf, \hat\qbf \in \simplex$, we have
  \begin{equation}
    \label{equation:supp:ent_grad}
    f(\qbf) = -L(\qbf, \qbf)
    \geq -L(\qbf, \hat\qbf)
    = -L(\hat\qbf, \hat\qbf) + \inpr{-\ellbf(\hat\qbf)}{\qbf - \hat\qbf}
    = f(\hat\qbf) + \inpr{-\ellbf(\hat\qbf)}{\qbf - \hat\qbf},
  \end{equation}
  where the inequality is thanks to the properness of~$\ellbf$.
  Thus,~$-\ellbf(\hat\qbf) \in \partial f(\hat\qbf)$ and \cref{equation:savage} hold for any~$\hat\qbf \in \simplex$.

  Conversely, suppose that there is a proper convex function~$f$ on~$\Rbb^N$ such that~$\domain f = \simplex$
  and, for all~$\hat\qbf \in \simplex$, there exists~$\hat\vbf \in \partial{f}(\hat\qbf)$ satisfying \cref{equation:savage}.
  Then, for~$\qbf \in \simplex$, we have~$L(\qbf, \qbf) = -f(\qbf)$ and
  \[
    L(\qbf, \qbf)
    = -f(\qbf)
    \leq -f(\hat\qbf) - \inpr{\hat\vbf}{\qbf - \hat\qbf}
    = L(\qbf, \hat\qbf)
    \quad \text{for } \hat\qbf \in \simplex,
  \]
  in turn,~$\ellbf$ is proper.

  Next, we show the equivalence of the strict properness of~$\ellbf$ and the strict convexity of~$f$ on~$\simplex$.
  Let~$f$ be a convex function on~$\simplex$ such that \eqref{equation:savage} holds.
  On the one hand, if~$\ellbf$ is strictly proper, then
  we have 
  \begin{align*}
    f((1-t)\qbf + t\qbf')
    &= L((1-t)\qbf + t\qbf', (1-t)\qbf + t\qbf') \\
    &= (1-t)L(\qbf, (1-t)\qbf + t\qbf') + tL(\qbf', (1-t)\qbf + t\qbf') \\
    &> (1-t)L(\qbf, \qbf) + tL(\qbf', \qbf')\\
    &=(1-t)f(\qbf, \qbf) + tf(\qbf', \qbf'),
  \end{align*}
  for~$\qbf, \qbf' \in \simplex$ and~$t \in (0, 1)$, proving the strict convexity of~$f$ on~$\simplex$.
  On the other hand, if~$\ellbf$ is not strictly proper, then there exist distinct~$\qbf, \qbf' \in \simplex$ such that 
  \[
    f(\qbf)
    = -L(\qbf, \qbf)
    = -L(\qbf, \qbf')
    = f(\qbf') + \inpr{\vbf'}{\qbf - \qbf'},
  \]
  where~$\vbf' \in \partial f(\qbf')$ satisfies \cref{equation:savage} such that~$L(\qbf, \qbf') = -f(\qbf') - \inpr{\vbf'}{\qbf - \qbf'}$.
  For any~$t \in (0,1)$, we have 
  \begin{align*}
    -L((1-t)\qbf + t\qbf', (1-t)\qbf + t\qbf')
    &= f((1-t)\qbf + t\qbf') \\
    &\leq (1-t)f(\qbf) + tf(\qbf') \\
    &= f(\qbf') + (1-t) \inpr{\vbf'}{\qbf - \qbf'} \\
    &= f(\qbf') + \inpr{\vbf'}{(1-t)\qbf + t\qbf'-\qbf'} \\
    &= -L((1-t)\qbf + t\qbf', (1-t)\qbf + t\qbf').
  \end{align*}
  This yields~$f((1-t)\qbf + t\qbf') = (1-t)f(\qbf) + f(\qbf')$, that is,~$f$ is not strictly convex on~$\simplex$.

  This completes the proof of the proposition.
\end{proof}

As a by-product of the proof of \cref{proposition:savage}, 
we obtain the following property thanks to the construction~\eqref{equation:generator}.
\begin{corollary}[Subgradient of conditional Bayes risk]
  \label{corollary:entropy_subgradient}
  For a regular proper loss $\ellbf$,
  define a proper convex function $f$ on $\Rbb^N$ by~\eqref{equation:generator}.
  Then, we have~$-\ellbf(\hat\qbf) \in \partial{f}(\hat\qbf)$ for~$\hat\qbf \in \simplex$.
  In addition,~$-\ellbf+\gamma\onebf$ remains to be a selector of~$\partial f$ for any~$\gamma:\simplex\to\Rbb$.
\end{corollary}
Let $\ellbf$ be regular and proper.
Then we observe from  inequalities~\eqref{equation:equivalence} and~\eqref{equation:supp:ent_grad} that 
\begin{equation}
  \label{equation:regret}
  R=B_{(f,\nabla f)}= B_{(f,-\ellbf+\gamma\onebf)}
    \end{equation}
for a function on $\gamma$ on $\simplex$.
This gives a closed form of a subgradient of~$-\underline{L}$ and is of interest per se.
Nevertheless, when one generates a proper loss from a proper convex function~$f$,
it is more standard via the Savage representation~\eqref{equation:savage} by~$\ell_y(\hat\qbf) = L(\ebf_y, \hat\qbf)$, where~$\ebf_y=[\delta_{1y} \; \cdots \; \delta_{Ny}]^\top$ is the standard basis encoding the label~$y\in[N]$.

\paragraph*{Initial examples.}
We quickly see some examples of multiclass proper losses to let readers familiarize with the definitions so far, and discuss more examples in \cref{section:examples}.

The first example is the log loss, $\ell_y(\qbf)=-\ln q_y$ for~$y\in[N]$.
This possibly takes~$\infty$ for~$q_y=0$, for which we need the regularity (see \cref{definition:regular}).
The log loss corresponds to the Kullback--Leibler divergence and Shannon entropy
\[
  R(\qbf,\hat\qbf) = D_{\text{KL}}(\qbf\|\hat\qbf) \defeq \sum_{y\in[N]} q_y\ln\left(\frac{q_y}{\hat q_y}\right),
  \quad
  \underline{L}(\qbf) = -\sum_{y\in[N]}q_y\ln q_y
\]
as the regret and the conditional Bayes risk, respectively.
Here, we have~$\nabla\underline{L}(\qbf) = \ellbf(\qbf) - \onebf$, which is equivalent to~$\ellbf(\qbf) \in \partial\underline{L}(\qbf)$ (in \cref{corollary:entropy_subgradient}; and \cref{section:convex_analysis} for equivalent subgradients).
Since each~$\ell_y$ depends solely on~$q_y$, this type of loss functions is called \emph{local}.
Indeed, the log loss is the only local proper loss \cite{Parry2012AOS}.
The locality is considered to be a desirable property in terms of interpretability \cite{Du2021}.

The second example is the Brier score \cite{Brier1950}:
\[
  \ell_y(\qbf) = \frac12\sum_{y'\in[N]}(\delta_{yy'} - q_{y'})^2 = -q_y + \frac{1+\|\qbf\|_2^2}{2}.
\]
This is no longer local as~$\ell_y(\qbf)$ depends on~$q_{y'}$ for~$y'\ne y$.
Interestingly, the lack of the locality has been reportedly relevant to the emergent ability of language models \cite{Du2024NeurIPS}.
The Brier score is strictly proper associated with the following regret and conditional Bayes risk
\[
  R(\qbf, \hat\qbf) = \frac12\|\qbf - \hat\qbf\|_2^2,
  \quad
  \underline{L}(\qbf) = \frac{1-\|\qbf\|_2^2}{2},
\]
which are the squared $2$-norm distance and negative squared $2$-norm, respectively.
Here, we have~$\nabla\underline{L}(\qbf) = -\qbf$, which is equivalent to~$\ellbf(\qbf) \in \partial\underline{L}(\qbf)$
because~$\ellbf(\qbf) = -\qbf + \gamma(\qbf)\cdot\onebf$ with the choice~$\gamma(\qbf)\defeq(1+\|\qbf\|_2^2)/2$.
Note that~$\nabla\underline{L}(\qbf) = -\qbf$ is not proper when regarded as a loss function \cite[\S4.1]{Gneiting:2007}; for this reason, we must always interpret the formula~$\ellbf\in\partial\underline{L}$ in \cref{corollary:entropy_subgradient} under the subgradient equivalence.

\subsection{Strongly proper losses}
For~$\kappa > 0$, a loss~$\ellbf$ is called~\emph{$\kappa$-strongly proper} if
\begin{equation}
  \label{equation:strongly_proper}
  R(\qbf, \hat\qbf) = L(\qbf, \hat\qbf) - \underline{L}(\qbf) \geq \frac{\kappa}{2}\|\qbf - \hat\qbf\|_2^2
  \quad \text{for $\qbf, \hat\qbf \in \simplex$.}
\end{equation}
Strongly proper losses have been introduced for~$N=2$ \cite{Agarwal:2014} and for general~$N \geq 3$ \cite{Zhang:2021} to derive a surrogate regret bound in the form of \eqref{equation:desiderata}.
For example, the log loss is~$1$-strongly proper \cite[Lemma~3]{Zhang:2021}.
For~$N=2$,~$\ellbf$ is regular and strongly proper if and only if its conditional Bayes risk~$-\underline{L}$ is strongly convex (with respect to the~$2$-norm) \cite[Theorem~10]{Agarwal:2014}.
As an immediate consequence of the inequality~\eqref{equation:strongly_proper}, we have the~$1/2$-order surrogate regret bounds for strongly proper losses:
\begin{equation}
  \label{equation:strongly_proper_regret_bound}
  \|\qbf - \hat\qbf\|_2 \leq \sqrt{\frac2\kappa R(\qbf, \hat\qbf)}
  \qquad \text{for any $\qbf,\hat\qbf \in \simplex$.}
\end{equation}

Several binary losses are shown to be strongly proper \cite[Table~1]{Agarwal:2014}.
Beyond binary losses, a similar bound to \eqref{equation:strongly_proper_regret_bound} has been known for \emph{Fenchel--Young losses} \cite{Blondel:2020} (which is relevant to proper losses but defined over ``dual'' points of $\hat\qbf\in\triangle^N$),
but requires a restrictive condition on $-\underline{L}$, Legendre-type.
See \cite[Lemma~3]{Blondel2019} and \cite[footnote~8]{Sakaue2024COLT} for details.
In the next section, we derive surrogate regret bounds for general multiclass proper losses.

\section{Regret bounds: Necessity of strict properness}
\label{section:regret_bounds}
In this section, we first study the moduli of convexity in \cref{section:moduli}.
Therein, we show that the strict convexity of a function is equivalent to the strict monotonicity of its modulus (\cref{proposition:moduli_monotonicity}), which ensures that its surrogate regret bound is non-vacuous.
Then, we show in \cref{section:regret_bounds_subsec} the surrogate regret bounds for general multiclass proper losses beyond strongly proper losses.
In \cref{section:downstream}, we relate surrogate regret bounds to several downstream tasks.
Readers who are interested in the benefits of surrogate regret bounds may refer to this section first.

\subsection{Moduli of convexity}
\label{section:moduli}
Before introducing the moduli of convexity, we study
the \emph{midpoint Jensen gap} of a convex function~$f: \simplex \to \Rbb$, which is defined by
\begin{equation*}
  J(\qbf, \check\qbf) \defeq \frac{f(\qbf) + f(\check\qbf)}{2} - f\left(\frac{\qbf + \check\qbf}{2}\right)
  \quad \text{for $\qbf, \check\qbf \in \simplex$.}
\end{equation*}
The midpoint Jensen gap is nonnegative by the convexity of~$f$ on~$\simplex$.
The midpoint Jensen gap is invariant under adding an affine function, and so is the modulus of convexity (defined later).
That is, the midpoint Jensen gaps of two convex functions~$f: \simplex \to \Rbb$ and~$f_{\lambda,\ubf}: \simplex \to \Rbb$ defined by
\[
  f_{\lambda,\ubf}(\qbf) \defeq f(\qbf) + \inpr{\ubf}{\qbf} + \lambda
  \quad \text{for $\qbf \in \simplex$}
\]
are the same, for any~$\ubf \in \Rbb^N$ and~$\lambda \in \Rbb$.
Moreover, we will show that for continuous convex functions~$f, g: \simplex \to \Rbb$,
their midpoint Jensen gaps are the same if and only if~$f - g$ is affine.
This property is reminiscent of the condition for the payoff equivalence \cite[Theorem~3]{McCarthy1956}
and the universal equivalence of surrogate losses \cite[Theorem~3]{Nguyen2009AoS} \cite[Theorem~1]{Duchi2018AoS}.
The proof is deferred to \cref{section:proofs}.
\begin{restatable}[Uniqueness up to affine functions]{proposition}{jensengapaffine}
  \label{proposition:affine}
  Let~$f, g: \simplex \to \Rbb$ be continuous convex functions.
  Then, their midpoint Jensen gaps are the same if and only if~$f - g$ is affine.
\end{restatable}

We extend the moduli of convexity defined on~$(\triangle^2,\|\cdot\|_1)$ \cite[Definition~4]{Bao:2023} to~$(\simplex,\|\cdot\|_p)$.
Note that the diameter of~$(\simplex,\|\cdot\|_p)$ is~$2^{1/p}$.
\begin{definition}[{Modulus of convexity~\cite{Borwein:2009}}]
  \label{definition:modulus}
  For a convex function~$f:\simplex \to \Rbb$, 
  its \emph{modulus of convexity} of~$f$ with respect to the~$p$-norm is the function~$\omega: [0, 2^{1/p}] \to [0, \infty)$ defined by
  \begin{equation*}
    \omega(r) \defeq \inf\setcomp{J(\qbf, \check\qbf)}{ \qbf, \check\qbf\in \simplex \text{ with } \|\qbf-\check\qbf\|_p\geq r}
      \quad \text{for }
    r\in [0, 2^{1/p}].
  \end{equation*}
\end{definition}
While the notion of moduli of convexity dates back to the classical literature on optimization \cite{Polyak1966} and Banach spaces \cite{Figiel:1976}, we view this as the smallest possible Jensen--Bregman divergence \cite{Nielsen2011} with the fixed~$p$-norm distance.
This idea is similar to variational problems for deriving tight Pinsker's inequalities \cite{Vajda1970} \cite{Fedotov2003}.

We will show that the convexity and strict convexity of a function are translated to the monotonicity and strict monotonicity of its modulus, respectively.
This is an important result throughout this article because the moduli of convexity characterize surrogate regret bounds, as we will see in \cref{section:regret_bounds_subsec} soon.
\begin{restatable}[Monotonicity of modulus]{theorem}{monotonicity}
  \label{proposition:moduli_monotonicity}
  Let~$f: \simplex \to \Rbb$ be a convex function.
  Then, the following assertions hold.
  \begin{enumerate}
    \item The modulus~$\omega$ is non-decreasing on~$[0, 2^{1/p}]$ and~$\omega(0) = 0$.
    \item $f$ is strictly convex on~$\simplex$ if and only if~$\omega$ is strictly monotone on~$[0, 2^{1/p}]$.
    \item $f$ is strongly convex on~$\simplex$ with respect to the~$p$-norm if and only if there exists~$\kappa>0$ such that~$\omega(r)\ge\kappa r^2$ on~$r\in[0,2^{1/p}]$.
  \end{enumerate}
\end{restatable}
From \cref{proposition:moduli_monotonicity}, the strong convexity of~$f$ is equivalent to the quadratic bound~$\omega(r)\gtrsim r^2$.
Thus, the modulus of convexity quantifies the convexity of a function.

Before proving \cref{proposition:moduli_monotonicity}, we show a lemma that is repeatedly used in the rest of the article.
This lemma guarantees that the infimum of the modulus of convexity $\omega$ is attainable, and the minimizer lies at the boundary of the constraint $\|\qbf-\check\qbf\|_p\ge r$ in its definition.
The proof is deferred to \cref{section:proofs}.
\begin{restatable}{lemma}{minimizer}
  \label{lemma:moduli_minimizer}
  Let~$f:\simplex \to \Rbb$ be a convex function.
  For~$r \in [0, 2^{1/p}]$, there exist~$\qbf^r, \check\qbf^r \in \simplex$ such that~$\omega(r) = J(\qbf^r, \check\qbf^r)$ and~$\|\qbf^r - \check\qbf^r\|_p=r$.
\end{restatable}

{\renewenvironment{proof}{\par\noindent{\bf Proof of \cref{proposition:moduli_monotonicity}\ }}{\hfill$\openbox$}
\begin{proof}
  Define
  \[
    \Dcal^N(r) \defeq \setcomp{(\qbf,\hat\qbf) \in \simplex \times \simplex}{\|\qbf - \hat\qbf\|_p \geq r}.
  \]
  For~$r', r \in [0, 2^{1/p}]$ with~$r' \le r$, we observe from the monotonicity~$\Dcal^N(r) \subseteq \Dcal^N(r')$ that~$\omega(r') \le \omega(r)$.
  It is easily seen that~$J(\qbf, \qbf) = 0$ holds for any~$\qbf \in \simplex$
  hence~$\omega(0)=0$.
  Thus, the first assertion follows.

  To show the second assertion, assume the strict convexity of~$f$ on~$\simplex$.
  Let~$r \in (0, 2^{1/p}]$.
  By \cref{lemma:moduli_minimizer}, there exist~$\qbf, \check\qbf \in \simplex$ such that~$\omega(r) = J(\qbf, \check\qbf)$ and~$\|\qbf-\check\qbf\|_p=r$.
  Define a curve $c:[0,1] \to \simplex$ by
  \[
    c(t) \defeq (1-t)\qbf + t\check\qbf \quad \text{for $t \in [0,1]$.}
  \]
  Since~$f \circ c$ is strictly convex on~$[0,1]$, we have the strict inequality
  \[
    \frac{f(c(\tau)) - f(c(0))}{\tau} < \frac{f(c(1)) - f(c(1-\tau))}{\tau}
    \quad \text{for $\tau \in (0, 1/2]$.}
  \]
  Consequently, we conclude 
  \[ 
    \omega((1-2\tau)r) \leq J(c(\tau), c(1-\tau)) < J(c(0),c(1)) = \omega(r)
    \quad \text{for $\tau \in (0, 1/2]$,}
  \]
  that is, the strict monotonicity of~$\omega$ on~$[0, 2^{1/p}]$.
  Conversely, if~$f$ is not strictly convex on~$\simplex$, there exist distinct~$\qbf, \check\qbf \in \simplex$ such that
  \begin{equation*}
    f((1-t)\qbf + t\check\qbf) = (1-t)f(\qbf) + tf(\check\qbf) \quad \text{for } t \in [0, 1].
  \end{equation*}
  This leads to~$J(\qbf, \check\qbf) = 0$.
  Consequently,~$\omega$ is not strictly increasing on~$[0, \|\qbf - \check\qbf\|_p]$.

  To show the third assertion, assume that~$f$ is~$\tilde\kappa$-strongly convex on~$\simplex$ with respect to the~$p$-norm for some~$\tilde\kappa>0$.
  By \cref{lemma:moduli_minimizer}, there exist~$\qbf, \check\qbf \in \simplex$ with~$\omega(r)=J(\qbf,\check\qbf)$ and~$\|\qbf-\check\qbf\|_p=r$.
  Then, the strong convexity of~$f$ with the choice~$t=1/2$ implies
  \[
    \omega(r) = \frac{f(\qbf)+f(\check\qbf)}{2} - f\left(\frac{\qbf+\check\qbf}{2}\right) \ge \frac{\tilde\kappa}{8}\|\qbf-\check\qbf\|_p^2 = \frac{\tilde\kappa}{8}r^2.
  \]
  Conversely, suppose that~$\omega(r)\ge\kappa r^2$ on~$r\in[0,2^{1/p}]$.
  This implies
  \begin{equation}
    \label{equation:supp:modulus_quad_lower_bound}
    J(\qbf,\check\qbf)=
    \frac{f(\qbf)+f(\check\qbf)}{2}-f\left(\frac{\qbf+\check\qbf}{2}\right)
    \ge \omega(\|\qbf-\check\qbf\|_p)
    \ge \kappa\|\qbf-\check\qbf\|_p^2
    \quad \text{for any~$\qbf,\check\qbf \in \simplex$.}
  \end{equation}
Taking $t\in[1/2,1)$,
we have
  \[
    \begin{aligned}
      f((1-t)\qbf+t\check\qbf)
      &= f\left((1-2t)\qbf + 2t\,\frac{\qbf+\check\qbf}{2}\right) \\
      &\le (1-2t)f(\qbf) + 2tf\left(\frac{\qbf+\check\qbf}{2}\right) \\
      &\le (1-2t)f(\qbf) + 2t\left[\frac{f(\qbf)+f(\check\qbf)}{2} - \kappa\|\qbf-\check\qbf\|_p^2\right] \\
      &= (1-t)f(\qbf) + tf(\check\qbf) - 2\kappa t\|\qbf-\check\qbf\|_p^2 \\
      &\le (1-t)f(\qbf) + tf(\check\qbf) - 2\kappa t(1-t)\|\qbf-\check\qbf\|_p^2,
    \end{aligned}
  \]
  where we used \eqref{equation:supp:modulus_quad_lower_bound}.
  Switching the roles of $\qbf$ and $\hat\qbf$ yields the $4\kappa$-strongly convexity of $f$ with respect to the~$p$-norm.

  All in all, the proof of the theorem is achieved.
\end{proof}
}

Despite the simple proof, this will lead to the necessity and sufficiency for a surrogate regret bound being non-vacuous in \cref{section:regret_bounds_subsec}, together with \cref{theorem:regret_lower_bound}.

\subsection{Surrogate regret bounds}
\label{section:regret_bounds_subsec}
Now, we give surrogate regret bounds with respect to the~$p$-norm.
The asymptotic behavior of a surrogate regret bound for a proper loss~$\ellbf$ is essentially governed by the modulus of convexity of (the negative of) its conditional Bayes risk~$-\underline{L}$.
This is an extension of surrogate regret bounds for binary classification \cite[Theorem~6]{Bao:2023} to multiclass classification.
Its proof (shown below) is an immediate generalization from \cite{Bao:2023}.
\begin{restatable}[Surrogate regret bounds]{theorem}{regretbound}
  \label{theorem:regret_lower_bound}
  Let~$\ellbf: \simplex \to [0, \infty]^N$ be a regular proper loss 
  and~$f:\Rbb^N \to (-\infty,\infty]$ a proper convex function defined by \eqref{equation:generator} with its modulus of convexity~$\omega$.
  For~$\qbf, \hat\qbf \in \simplex$, it holds
  \begin{equation}
    \label{equation:regret_lower_bound}
    \omega(\|\qbf - \hat\qbf\|_p) \le \frac{1}{2}R(\qbf, \hat\qbf),
  \end{equation}
  with the equality if~$\qbf=\hat\qbf$.
  If~$\ellbf$ is strictly proper additionally, then the equality of \eqref{equation:regret_lower_bound} holds if and only if~$\qbf=\hat\qbf$.
\end{restatable}
\begin{proof}
  By the definition of~$\omega$ together with \eqref{equation:regret}, it is sufficient to show
  \[
    J(\qbf, \hat\qbf) \leq \frac12 B_{(f,-\ellbf)}(\qbf\|\hat\qbf) 
    \quad \text{for $\qbf, \hat\qbf \in \simplex$.}
  \]
  By \cref{corollary:entropy_subgradient}, we have
  \[
    f\left(\frac{\qbf + \hat\qbf}{2}\right)
    \geq f(\hat\qbf) + \inpr{-\ellbf(\hat\qbf)}{\frac{\qbf + \hat\qbf}{2} - \hat\qbf}
    = f(\hat\qbf) + \frac{1}{2}\inpr{-\ellbf(\hat\qbf)}{\qbf - \hat\qbf},
  \]
  which implies
  \begin{equation*}
    \begin{aligned}
      J(\qbf, \hat\qbf)
      \le \frac{1}{2}\left[f(\qbf) - f(\hat\qbf) - \inpr{-\ellbf(\hat\qbf)}{\qbf - \hat\qbf}\right]
      = \frac{1}{2}B_{(f,-\ellbf)}(\qbf\|\hat\qbf).
    \end{aligned}
  \end{equation*}
  The equality can be seen immediately by choosing~$\hat\qbf = \qbf$.
  If~$\ellbf$ is strictly proper, then~$\omega$ is strictly monotone by \cref{proposition:moduli_monotonicity} and~$f$ is strictly convex, which indicates that the equality of \eqref{equation:regret_lower_bound} yields~$\qbf=\hat\qbf$.
\end{proof}

Let us discuss when a proper loss entails a \emph{non-vacuous} bound, %
which means that~$\|\qbf-\hat\qbf\|_p$ approaches zero whenever the surrogate regret~$R(\qbf,\hat\qbf)$ goes to zero.
If~$\ellbf$ is strictly proper, then~$\omega$ is strictly increasing, which in turn has an inverse function~$\omega^{-1}$ and leads \eqref{equation:regret_lower_bound} to
\begin{equation}
  \label{equation:regret_bound_inv}
  \|\qbf - \hat\qbf\|_p \le \begin{dcases}
    {\omega}^{-1}\left(\frac{1}{2}R(\qbf, \hat\qbf)\right) & \text{if~~$\frac12R(\qbf, \hat\qbf) \le \omega\Big(2^{\frac1p}\Big)$,} \\
    2^{\frac1p} & \text{otherwise.}
  \end{dcases}
\end{equation}
The strict monotonicity of~$\omega^{-1}$ is essential for non-vacuous bounds because we then have~$\rho\downarrow0$ if and only if~$\omega^{-1}(\rho)\downarrow0$.
Otherwise, we cannot always expect that the estimate~$\hat\qbf$ approaches~$\qbf$ even if the suboptimality~$R(\qbf, \hat\qbf)$ vanishes.
By \cref{proposition:savage,proposition:moduli_monotonicity}, the \emph{strict} properness of~$\ellbf$ is necessary and sufficient for the surrogate regret bound~\eqref{equation:regret_lower_bound} being non-vacuous.
This is why strict properness matters.

If~$\ellbf$ is strongly proper, the inequality~\eqref{equation:strongly_proper} implies that the negative Bayes risk~$-\underline{L}$ is strongly convex.
Combining the~$p$-norm bound~\eqref{equation:regret_bound_inv} with \cref{proposition:moduli_monotonicity}, we recover the~$1/2$-regret bound~\eqref{equation:strongly_proper_regret_bound} modulo a constant.

\paragraph*{Comparison: Pinsker's inequality.}
For illustration, let us consider the log loss under the binary case~$N=2$, where we identify~$[q\;1-q] \in \triangle^2$ with~$q\in[0,1]$.
Then, the generator function~$f$ defined by~\eqref{equation:generator} is the negative binary Shannon entropy~$f(q)=q\ln q+(1-q)\ln(1-q)$.
Its modulus of convexity (with~$1$-norm) admits the following form:
\[
  \omega(r) = \frac12\left[\left(1+\frac r2\right)\ln\left(1+\frac r2\right)+\left(1-\frac r2\right)\ln\left(1-\frac r2\right)\right].
\]
Note that this form coincides with the calibration function of the logistic loss \cite[Table~1]{Steinwart:2007}.
The~$p$-norm bound gives
\begin{equation}
  \label{equation:regret_bound_inv_kl}
  2^{\frac1p}|q-\hat q| \le \omega^{-1}\left(\frac12D_{\text{KL}}(q\|\hat q)\right)
  \quad \text{for $q,\hat q\in[0,1]$}.
\end{equation}
Note~$2^{1/p}|q-\hat q| = \|[q\;1-q]^\top - [\hat q\;1-\hat q]^\top\|_p$.
Moreover, we can verify~$\omega(r)\ge r^2/2$, which gives~$|q-\hat q|^2 \lesssim D_{\text{KL}}(q\|\hat q)$, namely, Pinsker's inequality.
The~$p$-norm bound~\eqref{equation:regret_bound_inv} can be viewed as generalizing Pinsker's inequality by allowing other Bregman divergences in the upper bound and the~$p$-norm distance in the lower bound.

\paragraph*{Comparison: surrogate regret bounds for margin-based losses.}
Margin-based classification \cite{Wang2024JMLR}, where a learner acts on~$\Rbb^N$-valued margin instead of a probabilistic estimate~$\hat\qbf\in\simplex$, is commonly used.
Let us consider binary classification based on the binary margin~$z \in \Rbb$.
Given a true probability~$q \in [0,1]$ (identified with~$[q\;1-q]^\top \in \triangle^2$) and margin~$z \in \Rbb$,
the classification performance is evaluated by the (conditional) 0-1 regret
\[
  \text{Reg}_{01}(q, z) \defeq [q\indicator{z\le 0} + (1-q)\indicator{z>0}] - \min\set{q, 1-q},
\]
where~$q$ and~$1-q$ indicates the class probabilities of~$y=1$ and~$y=2$, respectively.

In the binary case~$N=2$, we often use a (symmetric) margin-based losses~$\phi:\Rbb\to[0,\infty]$ as a surrogate loss, which operates on the binary margin~$z\in\Rbb$:
the logistic loss~$\phi_{\text{log}}(z)=\ln(1+\exp(-z))$ and the hinge loss~$\phi_{\text{hinge}}(z)=\max\set{0, 1-z}$ are common examples.
The prediction performances of the binary margin~$z$ for~$y=1$ and~$y=2$ are evaluated by~$\phi(z)$ and~$\phi(-z)$, respectively.
For binary margin-based losses, surrogate regret bounds with respect to~$\text{Reg}_{01}$ have been studied intensively \cite[Theorem~3]{Bartlett:2006}.
Here, we compare the 0-1 regret bounds and the $p$-norm regret bounds (\cref{theorem:regret_lower_bound}).
Let us denote the conditional risk \emph{of the binary margin~$z$} and Bayes risk given the true probability~$q$ by
\[
  L^{\text{mgn}}(q, z) \defeq q\phi(z) + (1-q)\phi(-z), \quad
  \underline{L}^{\text{mgn}}(q) \defeq \inf_{z \in \Rbb}L^{\text{mgn}}(q, z),
\]
respectively, and define~$\psi:[0,1] \to [0,\infty)$ by
\begin{equation}
  \label{equation:psi_transform}
  \psi(r) \defeq \inf\setcomp{L^{\text{mgn}}\left(\frac{1+r}{2}, z\right)}{z\in \Rbb, z\le 0} - \underline{L}^{\text{mgn}}\left(\frac{1+r}{2}\right).
\end{equation}
Then, for~$(q,z)\in[0,1]\times\Rbb$, we have
\begin{equation}
  \label{equation:01_regret_bound}
  \psi(\text{Reg}_{01}(q, z)) \le L^{\text{mgn}}(q, z) - \underline{L}^{\text{mgn}}(q).
\end{equation}
The function~$\psi$ is called~\emph{$\psi$-transform}, and later generalized as a calibration function \cite{Steinwart:2007}.
The~$\psi$-transform and the modulus of convexity~$\omega$ characterize the 0-1 regret bound~\eqref{equation:01_regret_bound} and the $p$-norm bound~\eqref{equation:regret_lower_bound}, respectively, and are closely related with each other.
The modulus of convexity is defined (in \cref{definition:modulus}) by the best possible Jensen gap~$J(\qbf,\hat\qbf)$ such that the true probability~$\qbf$ and estimated probability~$\hat\qbf$ are distant at least~$r$ in the sense of the~$p$-norm.
By contrast, the~$\psi$-transform in \eqref{equation:psi_transform} can be rewritten as follows:
\begin{equation}
  \label{equation:supp:psi_transform}
  \begin{aligned}
    \psi(r) &= \inf\setcomp{L^{\text{mgn}}\left(\frac{1+r}{2}, z\right) - \underline{L}^{\text{mgn}}\left(\frac{1+r}{2}\right)}{z\in\Rbb, z\le0} \\
    &= \inf\setcomp{L^{\text{mgn}}(q, z) - \underline{L}^{\text{mgn}}(q)}{\text{Reg}_{01}(q, z) \ge r},
  \end{aligned}
\end{equation}
which is the best possible surrogate regret~$L^{\text{mgn}}(q,z)-\underline{L}^{\text{mgn}}(q)$ such that the margin prediction~$z$ is suboptimal with respect to the true class probability~$q$ by the level~$r$ at least.
Therefore, both~$\psi$-transform and the modulus of convexity measures the best possible surrogate regret given a true probability and suboptimal prediction by the level~$r$ least.
Interested readers can refer to \cite[\S4.1]{Steinwart:2007} to find more details of \eqref{equation:supp:psi_transform}.

\paragraph*{Comparison: surrogate regret bounds for proper composite losses.}
Let us leave a remark on the existing surrogate regret bounds for proper \emph{composite} losses.
A proper composite loss \cite{Williamson:2016} are the composition of a proper loss~$\ell:\simplex\to[0,\infty]^N$ and an invertible link function~$\lambdabf:\simplex\to\Rbb^N$,~$\ellbf\circ\lambdabf^{-1}:\Rbb^N\to[0,\infty]^N$,
so that the composite loss can operate on a multiclass margin $\zbf\in\Rbb^N$ directly.
The cross-entropy loss is of this type, where~$\ellbf$ is the log loss and~$\lambdabf^{-1}$ is the softmax function:
\[
  \lambda_y(\zbf) = \frac{\exp(z_y)}{\displaystyle\sum_{i\in[N]}\exp(z_i)}.
\]
Prior to this article, a surrogate regret bound similar to \eqref{equation:regret_bound_inv} has been derived for proper composite losses, with the moduli of \emph{continuity} of the conditional risk~$L(\qbf, \cdot)$ \cite[Corollary~3]{Mey:2021}.
While the relationship between the moduli of convexity of~$-\underline{L}$ and the moduli of continuity of~$L(\qbf, \cdot)$ has not been clear,~$-\underline{L}$ suffices for deriving surrogate regret bounds
because a surrogate regret is solely determined by~$\underline{L}$ due to \eqref{equation:regret} and \cref{corollary:entropy_subgradient}.
Therefore, our surrogate regret bounds in \cref{theorem:regret_lower_bound} can be readily applied to proper composite losses.
Moreover, the existing surrogate regret bounds \cite[Corollary 3]{Mey:2021} has been limited to the binary case~$N=2$.
Our \cref{theorem:regret_lower_bound} is more general therein.

\subsection{Relating surrogate regret to downstream tasks}
\label{section:downstream}
The upper bound for the~$p$-norm~\eqref{equation:regret_bound_inv} is useful for many scenarios to assess the predictive performance of plug-in forecasters, i.e., post-processed forecasters based on the estimate~$\hat\qbf$.
Thus, we can regard the~$p$-norm bound as a \emph{versatile} surrogate regret bound across different downstream tasks.
Subsequently, we provide several examples of downstream tasks to support this idea.

\paragraph{Task 1: multiclass classification.}
Let us consider multiclass classification based on the post-process approach.
Given true and estimated probability vectors~$\qbf, \hat\qbf \in \simplex$, respectively, the plug-in forecaster based on the estimate~$\hat\qbf$ is given by~$\hat{y} \in \argmax_{y \in \Ycal} \hat q_y$, where the tie is broken arbitrarily.
Define~$\Lbf \in \Rbb^{N \times N}$ the 0-1 loss matrix by~$L_{ij} \defeq 1-\delta_{ij}$ for each~$(i,j)\in[N]^2$, and~$\Lbf_{y}$ denotes the~$y$-th column vector of~$\Lbf$.
Here, the forecaster's suboptimality in multiclass classification is measured by the (conditional) \emph{0-1 regret}
\begin{align*}
  \mathrm{Reg}_{01}(\qbf, \hat\qbf)
  &\defeq \sum_{n \in \Ycal} q_n(1-\delta_{n\hat{y}}) - \min_{y \in \Ycal} \sum_{n \in \Ycal} q_n(1-\delta_{ny}) \\
  &= \sum_{n \in \Ycal} q_n\left(L_{n\hat y} - \min_{y \in \Ycal} L_{ny}\right) \\
  &= \max_{y \in \Ycal} \sum_{n \in \Ycal} q_n\left(L_{n\hat y} - L_{ny}\right) \\
  &= \max_{y \in \Ycal} \inpr{\qbf}{\Lbf_{\hat{y}} - \Lbf_{y}},
\end{align*}
for~$\qbf,\hat\qbf\in\simplex$.
Let~$p^\ast$ denote the H\"{o}lder conjugate of~$p$.
The 0-1 regret can be bounded as
\[
  \mathrm{Reg}_{01}(\qbf, \hat\qbf)
  \leq \max_{y \in \Ycal} \inpr{\qbf - \hat\qbf}{\Lbf_{\hat{y}} - \Lbf_{y}}
  \leq \|\qbf - \hat\qbf\|_p \max_{y \in \Ycal} \|\Lbf_{\hat{y}} - \Lbf_{y}\|_{p^\ast}
  \leq 2^{1-\frac1p} \|\qbf - \hat\qbf\|_p,
\]
where the first inequality holds because~$\inpr{\hat\qbf}{\Lbf_{\hat{y}} - \Lbf_{y}} \leq 0$ for any~$y \in \Ycal$ attributed to the construction of~$\hat{y}$,
and the second inequality owes to H\"{o}lder's inequality.
Eventually, the 0-1 regret is controlled by the surrogate regret~$R(\qbf, \hat\qbf)$ via \eqref{equation:regret_bound_inv} if~$\ellbf$ is strictly proper,
which relates the estimation quality of~$\hat\qbf$ to the predictive performance of the post-processed forecaster via the~$p$-norm.

In case of binary classification, a closely related surrogate regret bound was presented \cite[Lemma~4]{Menon:2013} (but for class-imbalanced plug-in forecasters),
which has been later used to control the $1$-norm between the estimated and true class probabilities for F-measure optimization \cite{Koyejo:2014}.
Our result allows to generalize them to the multiclass case.

Note, however, that more direct control of the 0-1 regret is possible by
\[
  \Psi(\mathrm{Reg}_{01}(\qbf,\hat\qbf)) \le R(\qbf,\hat\qbf),
  \qquad \text{where~~} \Psi(r) = \underline{L}\left(\frac12\right) - \underline{L}\left(\frac12+r\right),
\]
which can be obtained via the second-order Taylor expansion of proper losses \cite[Corollary~27]{Reid:2011}.
In this case, we may obtain non-vacuous bounds even for non-strictly proper losses.
Our \cref{theorem:regret_lower_bound} does not provide such a tailored bound for the 0-1 regret but ``one-size-fits-all'' bounds for multiple tasks so that we can control the performance of other downstream tasks, not only classification.

\paragraph{Task 2: learning with noisy labels.}
Let us consider multiclass classification with class-conditional label noises:
a true label~$y$ is observed as~$\tilde y$ with probability~$C_{y\tilde y}$ with a row-stochastic noise matrix~$\Cbf \in [0,1]^{N\times N}$.
In this scenario, our access is limited to the noisy target probability vector~$\tilde\qbf = \Cbf^\top\qbf$, through which a noisy estimate~$\hat\qbf$ is obtained.
By following the noise-correction strategy \cite{Zhang:2021}, the plug-in forecaster based on the noisy estimate~$\hat\qbf$ is given by~$\check y \in \argmax_{y \in \Ycal}\check q_y$, where~$\check\qbf \defeq (\Cbf^\top)^{-1}\hat\qbf$ (provided that~$\Cbf$ is invertible).
Under this setup, the 0-1 regret of~$\check\qbf$ given~$\qbf$ is bounded as follows:
\[
  \begin{aligned}
    \mathrm{Reg}_{01}(\qbf, \check\qbf)
    &= \max_{y \in \Ycal} \inpr{\qbf}{\Lbf_{\check y} - \Lbf_{y}}\\
    &\leq \max_{y \in \Ycal} \inpr{\qbf-\check \qbf }{\Lbf_{\check y} - \Lbf_{y}}
    = \max_{y \in \Ycal} \inpr{\tilde\qbf - \hat\qbf}{\Cbf^{-1}(\Lbf_{\check y} - \Lbf_{y})} \\
    &\leq \|\tilde\qbf - \hat\qbf\|_p \max_{y \in \Ycal}\|\Cbf^{-1}(\Lbf_{\check y} - \Lbf_y)\|_{p^\ast},
  \end{aligned}
\]
where the first inequality holds because~$\inpr{\check\qbf}{\Lbf_{\check y} - \Lbf_y} \leq 0$ for any~$y \in \Ycal$ attributed to the construction of~$\check y$,
and the second inequality owes to H\"{o}lder's inequality.
The~$p$-norm~$\|\tilde\qbf - \hat\qbf\|_p$ can be minimized even with access to the noisy observations only, and the~$p$-norm bound~\eqref{equation:regret_bound_inv} controls this by the surrogate regret of a strictly proper loss.
This is also an extension of the previous surrogate regret transfer bounds \cite[Theorem~4]{Zhang:2021} beyond strongly proper losses.

\paragraph{Task 3: bipartite ranking.}
Consider~$N=2$ and identify~$\qbf = [q \; 1-q]^\top \in \triangle^2$ with the instance~$q \in [0,1]$.
Given two instances~$q, q' \in [0,1]$, we are interested in giving estimates~$\hat q, \hat q' \in [0,1]$ that yield a consistent ranking with~$(q,q')$.
In bipartite ranking, we use the estimates~$(\hat q, \hat q')$ directly without any post process.
The (conditional) \emph{ranking regret} \cite{Clemenccon:2008} \cite{Narasimhan:2013} is measured by
\[
  \mathrm{Reg}_\text{rank}(q, q', \hat q, \hat q') \defeq |q - q'| \left[ \indicator{(\hat q - \hat q')(q - q') < 0} + \frac12\indicator{\hat q = \hat q'} \right],
\]
where~$\indicator{A} = 1$ when the predicate~$A$ holds and~$0$ otherwise,
and the first and second terms penalize an inconsistent ranking and tie, respectively.
This can be immediately related to the~$1$-norm \cite{Agarwal:2014}:
\[
  \mathrm{Reg}_\text{rank}(q, q', \hat q, \hat q') \le |q - \hat q| + |q' - \hat q'|,
\]
where the bound~\eqref{equation:regret_bound_inv} can be further applied.
Thus, the ranking regret is controlled by the surrogate regret.%
\footnote{
  While we consider the plug-in approach to bipartite ranking, a number of studies have considered the pairwise ranking approach \cite{Agarwal2005JMLR} \cite{Kotlowski2011ICML}, taking the difference of two margin predictions optimized with a margin-based loss function.
  Interestingly, the surrogate regret bound unavoidably becomes vacuous when we use a restricted hypothesis space, as shown recently \cite{Mao2023ICML}.
}

\paragraph{Other benefits.}
In addition to the above examples, one can easily relate the $p$-norm and downstream tasks such as binary classification with generalized performance criteria \cite[(9)]{Kotlowski:2016}, which we omit here.
Another benefit of the $p$-norm bound~\eqref{equation:regret_bound_inv} is that it relates a possibly non-metric~$R$
to the $p$-norm.

To conclude this section, we raise attention to the kinship between moduli of convexity and the known devices such as calibration functions \cite{Bartlett:2006} \cite{Steinwart:2007} \cite{Osokin:2017} \cite{Bao:2020} \cite{Bao:2022}, comparison inequalities \cite{Mroueh:2012} \cite{Ciliberto:2020}, and Fisher consistency bounds \cite{Awasthi:2022:ICML} \cite{Awasthi:2022:NeurIPS} \cite{Mao:2023}.
In spite of the relevance, moduli of convexity are different in that these devices have been tailored for a specific target loss of each downstream task, whereas moduli are concerned with the $p$-norm.
For recent developments, see \cite{Cortes2025NeurIPS}---the author group has extended to $\Hcal$-consistency bounds accounting for hypothesis spaces. Interested readers may consult their prolific body of work on this topic.%
\footnote{We are grateful to a reviewer for alerting us to this remarkably recent body of work.}

\section{Lower bounds of surrogate regret order}
\label{section:regret_order}
We move on to the next main result: the surrogate regret order via the modulus of convexity cannot go beyond the square root.
To this end, we first review the Simonenko order function and the strong convexity used to establish the main result, and then show the main result.
In this section, let~$f: \simplex \to \Rbb$ be a convex function unless otherwise stated.

\subsection{Power evaluation of moduli}
To analyze how fast the surrogate regret bound~\eqref{equation:regret_lower_bound} can be, we analyze the behavior of the modulus~$\omega$ in terms of power functions.
To this end, we introduce the order of~$\omega$ below, which is well-defined since~$\omega(r) > 0$ for~$r \in (0, 2^{1/p}]$ from \cref{proposition:moduli_monotonicity}.
\begin{definition}[{Simonenko order function \cite{Simonenko:1964}}]
  \label{definition:simonenko_order}
  Let~$f:\simplex \to \Rbb$ be a strictly convex function
  and~$\omega:[0,2^{1/p}]\to[0,\infty)$ be the modulus of convexity of $f$.
  The \emph{Simonenko order function}~$\sigma: (0, 2^{1/p}] \to [0, \infty]$ (associated with~$\omega$) is defined by 
  \begin{equation*}
    \sigma(r) \defeq \frac{rD^-\omega(r)}{\omega(r)} \quad \text{for }r \in \big(0,2^{\frac1p}\big],
    \text{~~where~~}
    D^-\omega(r) \defeq
    \limsup_{\epsilon \downarrow 0} \frac{\omega(r) - \omega(r - \epsilon)}{\epsilon}.
  \end{equation*}
\end{definition}
The quantity~$D^-\omega$ is called the \emph{upper left Dini derivative} of~$\omega$ at~$r$.
If~$\omega$ is differentiable at~$r$, then~$D^-\omega(r) = \omega'(r)$ holds.
The Simonenko order function~$\sigma$ evaluates the order of~$\omega$.
Note that the following result holds for general continuous functions beyond moduli of convexity, but we restrict ourselves to moduli of convexity for brevity.
\begin{restatable}[Power evaluations of moduli]{proposition}{polynomialrate}
  \label{proposition:polynomial_rate}
  Let~$f: \simplex \to \Rbb$ be a strictly convex function.
  For a fixed~$r_0 \in (0, 2^{1/p}]$, define~$s, S \in[0,\infty]$ by
  \[
    s \defeq \inf_{r \in (0, r_0]} \sigma(r),
    \quad
    S \defeq \sup_{r \in (0, r_0]} \sigma(r), 
  \]
  and assume $S<\infty$.
  Then, the function~$r \mapsto \omega(r)r^{-s}$ is non-decreasing on~$(0,r_0)$ and 
the function~$r \mapsto \omega(r)r^{-S}$ is non-increasing on~$(0,r_0)$.
  Moreover, the following inequalities hold for any~$r \in [0,r_0]$:
  \[
    \left[\frac{\omega(r_0)}{r_0^{S}}\right]r^{S}
    \leq \omega(r)
    \leq \left[\frac{\omega(r_0)}{r_0^{s}}\right]r^{s}.
  \]
\end{restatable}

Roughly speaking, \cref{proposition:polynomial_rate} provides us $r^S\lesssim\omega(r)\lesssim r^s$ as $r\downarrow0$.
This order evaluation is useful when analyzing the asymptotic convergence rate of the surrogate regret bound~\eqref{equation:regret_lower_bound}.
\cref{proposition:polynomial_rate} is easily proved when $\omega$ is differentiable.
Indeed, if $\omega$ is differentiable,
then it holds for $t\in (0,r_0)$ that 
  \[
\frac{s}{t}=  \frac{1}{t}\inf_{r \in (0, r_0]} \sigma(r)
    \le \frac{\omega'(t)}{\omega(t)}
    \le \frac{1}{t}\sup_{r \in (0, r_0]} \sigma(r)
=\frac{S}{t}    
  \]
by the definition~$s$ and~$S$,
and integrating it on~$[r,r']\subset [0, r_0]$ gives 
  \[
s\ln \frac{r'}{r}
\le \int_r^{r'}\frac{\omega'(t)}{\omega(t)}\rd{t}
    = \ln\frac{\omega(r')}{\omega(r)}
\leq
S\ln\frac{r'}{r},
\]
which is equivalent to the desired monotonicity.
Thus, the~$p$-norm upper bound~\eqref{equation:regret_bound_inv} is controlled by the rate~$\omega^{-1}(\rho) = O(\rho^{1/S})$ as~$\rho \downarrow 0$.
Since we are interested in the behavior of~$\|\qbf-\hat\qbf\|_p$ when~$\hat\qbf$ is close to the minimizer of~$R(\qbf,\cdot)$, we focus on the asymptotic behavior of the Simonenko order function as~$r \downarrow 0$.

The complete proof of \cref{proposition:moduli_monotonicity} for non-differentiable~$\omega$ is deferred to \cref{section:proof:power_evaluation}.

\subsection{Strong convexity and its relation to moduli}
For the asymptotic analysis of~$\sigma$, we leverage strong convexity \cite[\S2.1.3]{Nesterov2013}.
Herein, we define the strong convexity parameter for a %
convex function~$f: \simplex \to \Rbb$ and~$t \in (0, 1)$ by
\[
  \begin{aligned}
  \kappa^{f,t}_p  &\defeq \ \inf\setcomp{
    \frac{2\left[ (1-t) f(\qbf) + tf(\check\qbf) -f((1-t)\qbf+t\check\qbf)\right]}{t(1-t)\|\qbf-\check\qbf\|_p^2}
  }{
    \text{distinct } \qbf, \check\qbf \in \simplex
    }, \\
    \kappa^{f}_{p} &\defeq \inf_{t \in (0,1)} \kappa^{f,t}_{p}.
  \end{aligned}
\]
We observe from the convexity of~$f$ that~$\kappa^f_p\in [0,\infty)$ and
\[
  f((1-t)\qbf+t\check\qbf) \leq (1-t) f(\qbf)+ tf(\check\qbf)-\frac{\kappa^f_p}{2}t(1-t)\|\qbf-\check\qbf\|_p^2
  \quad \text{for $\qbf, \check\qbf \in \simplex$ and $t \in (0,1)$.}
\]
When~$p=2$ and~$f$ is twice continuously differentiable over~$\simplex_+$, $\Hess f$ (the Hessian of $f$) satisfies~$\Hess f-\kappa_2^f\Ibf_N\succeq \Obf$ \cite[Theorem~2.1.10]{Nesterov2013}, where~$\Ibf_N$ is the identity matrix.%
\footnote{
Remark that the strong convexity parameter depends on the underlying set where we take the infimum.
Let us define the strong convexity parameter on~$\Rbb^N$ by replacing~$\simplex$ with~$\Rbb^N$ and write~$\bar\kappa^{f,t}_p$ instead of $\kappa^{f,t}_p$ for each $t\in (0,1)$.
By~$\simplex \subset \Rbb^N$, we observe that
$
  \bar\kappa^{f,t}_p \leq \kappa^{f,t}_p \text{~~for a function $f: \Rbb^N \to (-\infty, \infty]$,}
$
where the equality does not necessarily hold.
}

To calculate $\kappa^{f}_p$, we only need to know $\kappa^{f,1/2}_p$.
The proof is deferred to \cref{section:proofs}.
\begin{restatable}[Strong convexity parameter at midpoint]{proposition}{modulusmidpoint}
  \label{proposition:sc_modulus_midpoint}
  For a continuous convex function~$f: \simplex \to \Rbb$, it holds~$\kappa^f_p=\kappa^{f,1/2}_p$.
\end{restatable}

The strong convexity parameter~$\kappa^f_p$ is equivalent up to constant across different~$p\ge 1$, which can be seen as follows.
For the midpoint Jensen gap~$J$ (defined in \cref{section:moduli}), since we have
\begin{equation}
  \label{equation:sc_modulus_local_midpoint}
  \kappa^{f,\frac12}_p = \inf\setcomp{
    \frac{8J(\qbf,\check\qbf)}{\|\qbf-\check\qbf\|_p^2}
  }{
     \text{distinct $\qbf, \check\qbf \in \simplex$}
  },
\end{equation}
the following bound holds:
\begin{align*}
  \inf\setcomp{\frac{\|\qbf\|_p^2}{\|\qbf\|_2^2}}{\qbf \in \simplex}
  \leq
  \frac{\kappa^{f,\frac12}_2}{\kappa^{f,\frac12}_p}
  \leq \sup\setcomp{\frac{\|\qbf\|_p^2}{\|\qbf\|_2^2}}{\qbf \in \simplex}.
\end{align*}
In particular, if~$N=2$, we have
\[
  \|\qbf-\check\qbf\|_p
  = \left[(q_1-\check q_1)^p+(q_2-\check q_2)^p\right]^{\frac1p}
  = \left[(q_1-\check q_1)^p+(q_1-\check q_1)^p\right]^{\frac1p}
  = 2^{\frac1p}|q_1-\check q_1|
\]
for~$\qbf, \check\qbf \in \triangle^2$, 
and hence,~$2\kappa^{f,1/2}_2 = 2^{2/p}\kappa^{f,1/2}_p$.
Therefore,~$\kappa_p^f$ remains the same up to constant regardless of the choice of~$p\geq1$.

We will use the following representation of $\kappa_p^{f,1/2} = \kappa_p^f$ (given by \cref{proposition:sc_modulus_midpoint}) repeatedly later, which follows by definition of the modulus of convexity (defined in \cref{definition:modulus}) and~\eqref{equation:sc_modulus_local_midpoint}:
\begin{equation}
  \label{equation:sc_modulus_local}
  \kappa^{f}_p
  = \inf\setcomp{\frac{8\omega(r)}{r^2}}{r \in (0,2^{\frac1p}]}.
\end{equation}

\subsection{Asymptotic lower bound}
The asymptotic behavior of the Simonenko order~$\sigma(r)$ as~$r \downarrow 0$ is controlled by the strong convexity parameter~$\kappa^f_p$.
We consider a ``local'' version of the strong convexity parameter.
\begin{definition}[Local strong convexity modulus]
  \label{definition:local_modulus}
  For a convex function~$f:\simplex \to \Rbb$,
  define~$K_p^f:(0,2^{1/p}] \to \Rbb$ by
  \[
    K^{f}_{p}(r) \defeq \frac{8\omega(r)}{r^2}.
  \]
\end{definition}
This quantity is defined based on the alternative expression of the strong convexity parameter~$\kappa^f_p$ in \eqref{equation:sc_modulus_local}.
From the relationship~\eqref{equation:sc_modulus_local},~$K^{f}_{p}(r) \geq \kappa^{f}_p$ always holds on~$r \in (0,2^{1/p}]$.
We show that~$K^f_p$ is lower semi-continuous and left-continuous (but not continuous in general without additional assumptions) in \cref{section:property_local_modulus}.

Now, we analyze the asymptotic behavior of the moduli of convexity~$\sigma(r)$ at~$r \downarrow 0$ when the Bregman generator~$f$ is continuous on~$\simplex$, which is our second main result.
Therein, we assume the continuity of~$f$ to prevent~$f$ from being discontinuous on the~$\simplex\setminus\simplex_+$.
\begin{restatable}[Lower bound of order]{theorem}{regretorder}
  \label{theorem:regret_order}
  Let $f:\simplex \to \Rbb$ be a continuous strictly convex function.
  Assume one of the following two conditions:
  \begin{enumerate}
    \renewcommand{\theenumi}{C\arabic{enumi}}
    \renewcommand{\labelenumi}{\textup{(\theenumi)}}

    \item \label{theorem:condition:modulus_positive}
    $\kappa^f_p > 0$.
    
    \item \label{theorem:condition:modulus_conti}
    $K^f_p$ is continuous on $(0,r_0]$ for some $r_0 \in (0,2^{1/p}]$ and $K^f_p$ converges as $r \downarrow 0$.
  \end{enumerate}
  Then,
  \begin{equation}
    \label{equation:order_limsup}
    \limsup_{r \downarrow 0} \sigma(r) \geq 2.
  \end{equation}
  Moreover, if we assume both conditions, then 
  \begin{equation}
    \label{equation:order_liminf}
    \liminf_{r \downarrow 0} \sigma(r) \geq 2.
  \end{equation}
\end{restatable}
Let us discuss the applicability of \cref{theorem:regret_order}.
First,~$f$ is assumed to be strictly convex, which means that we deal with a strictly proper loss through the strictly convex negative Bayes risk~$f=-\underline{L}$ in \eqref{equation:generator}.
In \cref{theorem:regret_order}, we additionally require either \eqref{theorem:condition:modulus_positive} of \eqref{theorem:condition:modulus_conti}.
The condition \eqref{theorem:condition:modulus_positive} assumes nothing else but the strong convexity of~$f$.
In other words, \eqref{theorem:condition:modulus_positive} assumes that the underlying proper loss is strongly proper.
Indeed, the strong convexity parameter~$\kappa_p^f > 0$ is equivalent to the modulus~$\kappa$ defining strongly proper losses in \eqref{equation:strongly_proper}.
It is more interesting when \eqref{theorem:condition:modulus_positive} does not hold but \eqref{theorem:condition:modulus_conti} holds, where the underlying loss is strictly proper but no longer strongly proper.
The continuity assumption of~$K_p^f$ is mild enough to cover many reasonable examples of the negative Bayes risk~$f$, as we will see in \cref{section:examples}.

It follows from \cref{proposition:polynomial_rate} and \cref{theorem:regret_order} that the~$p$-norm bound~\eqref{equation:regret_bound_inv} is controlled by the rate of~$\omega^{-1}(\rho)$ cannot be faster than~$O(\rho^{1/2})$ for a strictly proper~$\ellbf$ satisfying either \eqref{theorem:condition:modulus_positive} or \eqref{theorem:condition:modulus_conti}.
To see this, we invoke \cref{proposition:polynomial_rate} to observe that for a fixed~$r_0 \in (0, 2^{1/p}]$,
\[
  \omega^{-1}(\rho) \leq r_0 \omega(r_0)^{-\frac{1}{S}} \cdot \rho^{\frac{1}{S}}
  \quad \text{for any $r \in [0, r_0]$ such that $\rho = \omega(r)$.}
\]
If the bound \eqref{equation:order_limsup} holds, then we have
\[
  2 \leq \limsup_{r \downarrow 0} \sigma(r) \leq \sup_{r \in (0, r_0]} \sigma(r) = S,
\]
which implies~$\rho^{1/S} \geq \rho^{1/2}$ (for~$\rho < 1$).
Thus, we discern the optimal rate~$\omega^{-1}(\rho) = O(\rho^{1/2})$ for~$\rho \in [0, 1]$.
This assures that strongly proper losses asymptotically achieve the optimal rate~$O(\rho^{1/2})$ as seen in \eqref{equation:strongly_proper_regret_bound}.
Moreover, the optimal rate~$\omega^{-1}(\rho)=O(\rho^{1/2})$ penetrates into a broad family of strictly proper losses.

To prove \cref{theorem:regret_order}, we leverage the following lemma to locally control~$\sigma(r)$, which is proven in \cref{section:proofs}.
\begin{restatable}{lemma}{moduluseval}
  \label{lemma:modulus_evaluation}
  Let $f:\simplex \to \Rbb$ be a continuous convex function.
  Then, for any $\qbf, \check\qbf \in \simplex$ and $r \in (0, 2^{1/p}]$,
  \begin{align*}
    \liminf_{r \downarrow 0} K^f_p(r) = \kappa^f_p, \quad
    J(\qbf,\check\qbf) \geq \frac{\kappa^f_p}{8}\|\qbf-\check\qbf\|_p^2, \quad \text{and} \quad
    D^-\omega(r) \geq \frac{\kappa^f_p}{4}r.
  \end{align*}
\end{restatable}

{
\renewenvironment{proof}{\par\noindent{\bf Proof of \cref{theorem:regret_order}\ }}{\hfill$\openbox$}
\begin{proof}
We observe from the strict convexity of~$f$ and \cref{proposition:moduli_monotonicity} that~$K^f_p > 0$ on~$r \in (0,2^{1/p}]$.
By assuming \eqref{theorem:condition:modulus_positive} only, it follows from \cref{lemma:modulus_evaluation} that 
\begin{equation}
  \label{equation:proof:lower_bound_order_function}
  \limsup_{r \downarrow 0} \sigma(r)
  = \limsup_{r\downarrow 0} \frac{rD^-\omega(r)}{\omega(r)}
  \geq \limsup_{r \downarrow 0} \frac{r\cdot\dfrac{\kappa^f_p}{4}r}{\dfrac{K^f_p(r)}{8}r^2}
  = \limsup_{r \downarrow 0} \frac{2\kappa^f_p}{K^f_p(r)}
  = 2.
\end{equation}
In addition, assume \eqref{theorem:condition:modulus_conti} together.
Then, in the similar manner to \cref{equation:proof:lower_bound_order_function}, we have
\begin{align*}
  \liminf_{r \downarrow 0} \sigma(r)
  \geq \liminf_{r \downarrow 0} \frac{2\kappa^f_p}{K^f_p(r)}
  = \lim_{r \downarrow 0} \frac{2\kappa^f_p}{K^f_p(r)}
  = 2.
\end{align*}

Next, assume \eqref{theorem:condition:modulus_conti} only, and~$\kappa^f_p > 0$ does not hold.
In this case, \cref{lemma:modulus_evaluation} indicates that~$K^f_p(r) \downarrow 0$ as~$r \downarrow 0$, from which with the intermediate value theorem,
we can inductively define~$(r_j)_{j \in \Nbb} \subset (0,r_0]$ by 
\[
  r_j \defeq \inf\setcomp{r \in (0,2^{1/p}]}{K^f_p(r) \geq \frac12 K^f_p(r_{j-1})}.
\]
Then, $(r_j)_{j \in \Nbb}$ converges to $0$ because $K^f_p > 0$ always holds on~$r \in (0,2^{1/p}]$ and hence~$K^f_p(r) = 0$ if and only if~$r = 0$.
We see %
\[
  K^f_p(r) < \frac{1}{2}K^f_p(r_{j-1}) = K^f_p(r_j)
  \quad \text{for $r\in (0,r_j]$.}
\]
This yields
\begin{align*}
  D^-\omega(r_j)
  = \limsup_{\epsilon \downarrow 0} \frac{\omega(r_j) - \omega(r_j-\epsilon)}{\epsilon}
  &= \limsup_{\epsilon \downarrow 0} \frac{\dfrac{K^f_p(r_j)}{8}r_j^2-\dfrac{K^f_p(r_j-\epsilon)}{8}(r_j-\epsilon)^2}{\epsilon}\\
  &\geq \limsup_{\epsilon \downarrow 0} \frac{\dfrac{K^f_p(r_j)}{8}r_j^2-\dfrac{K^f_p(r_j)}{8}(r_j-\epsilon)^2}{\epsilon}
  = \frac{K^f_p(r_j)}{4}r_j,
\end{align*}
which implies 
\begin{align*}
  \limsup_{r \downarrow 0}\sigma(r)
  \geq \limsup_{j \to \infty} \sigma(r_j)
  =\limsup_{j \to \infty} \frac{r_jD^- \omega(r_j)}{\omega(r_j)}
  \geq \limsup_{j \to \infty} \frac{r_j \cdot \dfrac{K^f_p(r_j)}{4}r_j}{\dfrac{K^f_p(r_j)}{8}r_j^2}
  = 2.
\end{align*}

Thus the proof of \cref{theorem:regret_order} is completed.
\end{proof}
}

\paragraph*{Lower bound of $\omega^{-1}(\rho)$.}
\Cref{proposition:polynomial_rate} also implies that for some $r_0 \in (0, 2^{1/p}]$,
\[
  \omega^{-1}(\rho) \geq [r_0\omega(r_0)^{-\frac{1}{s}}] \cdot \rho^{\frac{1}{s}}
  \quad \text{for $r \in [0, r_0]$ such that $\rho = \omega(r)$.}
\]
If the bound \eqref{equation:order_liminf} holds with the \emph{strict} inequality, then we can choose $r_0$ to satisfy
\[
  \frac{2 + \varsigma}{2} \leq \inf_{r \in (0, r_0]} \sigma(r) = s
  \quad \text{for~~$\varsigma \defeq \liminf_{r \downarrow 0} \sigma(r) > 2$,}
\]
which implies $\rho^{1/s} \geq \rho^{2/(2+\varsigma)} > \rho^{1/2}$ (for~$\rho < 1$).
Thus,~$\omega^{-1}$ admits the lower bound~$\Omega(\rho^{1/2})$ for~$\rho \in [0, 1]$.

\paragraph*{Comparison with the known lower bound.}
A relevant lower bound $\omega^{-1}(\rho) = \Omega(\rho^{1/2})$ has been shown previously for margin-based losses \cite[Theorem~4]{Frongillo:2021} \cite[Theorem~4.2]{Mao2024NeurIPS}.
These lower bounds assume that a loss is strongly convex and has a locally Lipschitz gradient \cite[Assumption~1]{Frongillo:2021} or a loss satisfies a relaxed version of the strict convexity \cite[Theorem~4.2]{Mao2024NeurIPS}.
These conditions are assumed under the loss differentiability, whereas both our \eqref{theorem:condition:modulus_positive} and \eqref{theorem:condition:modulus_conti} do not need the differentiability of $\ellbf$.
Ergo, the differentiability assumption is lifted to show the optimality of $\omega^{-1}(\rho) = O(\rho^{1/2})$.
To show~$\omega^{-1}(\rho) = \Omega(\rho^{1/2})$ in our case, \eqref{theorem:condition:modulus_positive} and \eqref{theorem:condition:modulus_conti} coupled with the existence of the limit of~$K^f_p(r)$ as $r \downarrow 0$ suffice.

As a side note, convolutional Fenchel--Young losses have been recently proposed to circumvent these commonly known squared-root lower bounds for convex smooth surrogate losses~\cite{Cao2025}.
Their surrogate regret bounds are derived for the multiclass 0-1 loss (and its generalization), which is clearly different from our the $p$-norm distance.
Thus, it does not contradict with the square-root lower bounds of \cref{theorem:regret_lower_bound}.

\begin{figure}[t]
  \centering
  \includegraphics[width=0.5\textwidth]{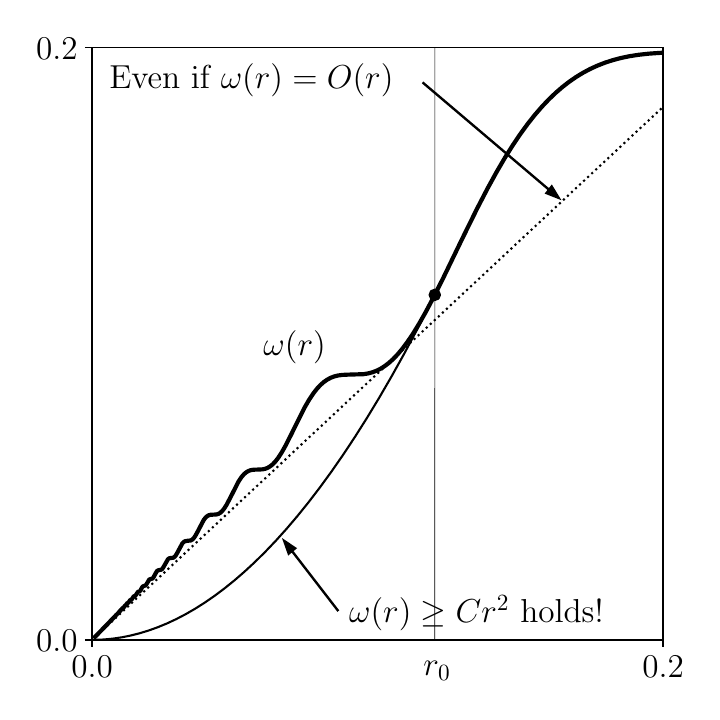}
  \caption{
    Illustration of $\omega(r) = r\sin\left(\frac1r\right) - \mathrm{Ci}\left(\frac1r\right) + r$.
    This~$\omega$ is asymptotically linear, but can only have a slower finite-range bound~$\omega(r)\gtrsim r^2$ than the linear rate.
  }
  \label{figure:counterexample}
\end{figure}

\begin{remark}
  Our analysis with the Simonenko order evaluates the order of~$\omega$ by a power function in the form of~$r^S \lesssim \omega(r) \lesssim r^s$ for~$r \in [0,2^{1/p}]$.
  This is an evaluation for a \emph{finite} range, which is more than an \emph{asymptotic} evaluation.
  Despite its subtlety, it often matters, as seen in the following example:
  \[
    \omega(r) = r\sin\left(\frac1r\right) - \mathrm{Ci}\left(\frac1r\right) + r \quad \text{for $r > 0$},
    \quad \text{where} \quad \mathrm{Ci}(z) \defeq -\int_z^\infty t^{-1}\cos(t)\rd{t}.
  \]
  This~$\omega$ is monotonically increasing and satisfies~$\omega(r) = O(r)$ as~$r\downarrow0$.
  However, when it comes to the finite evaluation, we cannot go faster than~$\omega(r) \gtrsim r^2$
  just because this~$\omega$ satisfies \eqref{theorem:condition:modulus_conti} and \cref{theorem:regret_order} implies~$\limsup_{r \downarrow 0}\sigma(r) \geq 2$.
  Thus, the finite evaluation gives a better characterization when we assess the convergence rate of finitely large surrogate regret.
  See \cref{figure:counterexample} to better understand the above example.
\end{remark}

\section{Examples}
\label{section:examples}
We overview a couple of proper losses.
Since one can generate a proper loss~$\ellbf$ from a convex function~$f=-\underline{L}$ thanks to the Savage representation~\eqref{equation:savage}, we show examples in terms of the corresponding convex functions in \cref{table:examples}.
To facilitate closed-form solutions of~$\omega$, we restrict ourselves to $N=2$ and $p=1$.
\Cref{table:examples} lists several convex functions with their moduli, whose derivations are given previously \cite{Bao:2023}.

\begin{table}[t]
  \caption{
    Examples of a convex function $f$.
    For $\omega$, we show the expressions with $(N,p)=(2,1)$.
  }
  \label{table:examples}
  \centering
  \small
  \renewcommand{\arraystretch}{1.6}
  \begin{tabular}{llllc}
    \toprule
    {} & $f(\qbf)$ & $\alpha$ & Modulus $\omega(r)$ & Loss $\ellbf$ \\
    \midrule
    Shannon & $\inpr{\qbf}{\ln\qbf}$ & --- & \normalsize{$\frac{1}{2}[(1+\frac r2)\ln(1+\frac r2) + (1-\frac r2)\ln(1-\frac r2)]$} & \footnotesize{Log}
    \\[4pt]
    \specialrule{0.1pt}{0pt}{0pt}
    \multirow{2}{*}{$\|\cdot\|_\alpha^2$} & \multirow{2}{*}{\Large $\frac{\|\qbf\|_\alpha^2-1}2$} & \footnotesize{$1 < \alpha < 2$} & $\frac{1}{8}\left\|\vecarr{1+r/2}{1-r/2}\right\|_\alpha^2 - 2^{2/\alpha-3}$ & \multirow{2}{*}{\footnotesize\renewcommand{\arraystretch}{0.95}\begin{tabular}{c}Brier\\($\alpha=2$)\end{tabular}}
    \\[6pt]
    {} & {} & \footnotesize{$2 \leq \alpha$} & $\frac{1}{4}\left\|\vecarr{r/2}{1-r/2}\right\|_\alpha^2 - \frac{1}{8}\left\|\vecarr{r/2}{2-r/2}\right\|_\alpha^2 + \frac{1}{4}$ & {}
    \\[4pt]
    \specialrule{0.1pt}{0pt}{0pt}
    \multirow{3}{*}{$\|\cdot\|_\alpha$} & \multirow{3}{*}{\Large $\frac{\|\qbf\|_\alpha-1}{\alpha-1}$} & \footnotesize{$1 < \alpha \leq 3/2$} & $\frac{1}{2(\alpha-1)}\left(\left\|\vecarr{1+r/2}{1-r/2}\right\|_\alpha - 2^{1/\alpha}\right)$ & \multirow{3}{*}{\footnotesize\renewcommand{\arraystretch}{0.95}\begin{tabular}{l}Pseudo-\\spherical\end{tabular}} \\[4pt]
    {} & {} & \footnotesize{$3/2 < \alpha < 2$} & (No closed-form in general) & {} \\[4pt]
    {} & {} & \footnotesize{$2 \leq \alpha$} & $\frac{1}{2(\alpha-1)}\left(\left\|\vecarr{r/2}{1-r/2}\right\|_\alpha - \left\|\vecarr{r/2}{2-r/2}\right\|_\alpha + 1\right)$ & {}
    \\[4pt]
    \specialrule{0.1pt}{0pt}{0pt}
    \multirow{3}{*}{Tsallis} & \multirow{3}{*}{\Large $\frac{\|\qbf\|_\alpha^\alpha-1}{\alpha-1}$} & \footnotesize{$1<\alpha<2$} & \multirow{2}{*}{\normalsize $\frac{1}{2^\alpha(\alpha-1)}\left(\left\|\vecarr{1+r/2}{1-r/2}\right\|_\alpha^\alpha - 2\right)$} & \multirow{3}{*}{\normalsize{$\alpha$-log}} \\
    {} & {} & \footnotesize{$3<\alpha$} & {} & {} \\[4pt]
    {} & {} & \footnotesize{$2 \leq \alpha \leq 3$} & \normalsize{$\frac{1}{2(\alpha-1)}(\left\|\vecarr{r/2}{1-r/2}\right\|_\alpha^\alpha - \frac{1}{2^{\alpha-1}}\left\|\vecarr{r/2}{2-r/2}\right\|_\alpha^\alpha + 1)$} & {}
    \\
    \bottomrule
  \end{tabular}
\end{table}
\begin{figure*}[t]
  \centering
  \resizebox{0.48\textwidth}{!}{\input{figures/mod1}} \hfill
  \resizebox{0.435\textwidth}{!}{\input{figures/mod2}}
  \caption{Numerical plots of $K^f_p(r) = 8\omega(r)/r^2$ for each $f$ in \cref{table:examples}.}
  \label{figure:modulus}
\end{figure*}

\paragraph*{Log loss.}
If we choose
\[
f(\qbf)=\inpr{\qbf}{\ln\qbf}=\sum_{n\in[N]}q_n\ln q_n,
\]
where~$\ln$ is applied in the element-wise manner,
we can generate the log loss~$\ell_y(\qbf)=-\ln q_y$.
In the binary case, we have
\[
  \ell_1(\qbf)=-\ln q \quad \text{and} \quad
  \ell_2(\qbf)=-\ln(1-q)
  \quad \text{for $\qbf=[q\;1-q]^\top\in\triangle^2$}.
\]
As we saw in \cref{section:savage} and \cref{section:regret_bounds_subsec}, the log loss is associated with the Kullback--Leibler divergence, and its surrogate regret bound slightly improves Pinsker's inequality yet is asymptotically equivalent.

To see the asymptotic speed of the~$1$-norm bound~\eqref{equation:regret_bound_inv} for the log loss, we investigate the power evaluation of~$\omega^{-1}(\rho) \lesssim \rho^{1/S}$ through \cref{theorem:regret_order},
where~$S$ is given in \cref{proposition:polynomial_rate}.
We can see that both \eqref{theorem:condition:modulus_positive} and \eqref{theorem:condition:modulus_conti} are satisfied in this case.
Indeed, the continuity of~$K_p^f$ is obvious, and
\[
  \begin{aligned}
    \lim_{r\downarrow0} K_p^f(r)
    &= 4\lim_{r\downarrow0} \frac{(1+\frac r2)\ln(1+\frac r2)+(1-\frac r2)\ln(1-\frac r2)}{r^2} \\
    &= \lim_{r\downarrow0} \left(\frac{1}{2+r}+\frac{1}{2-r}\right) \\
    &= 1 < \infty,
  \end{aligned}
\]
where L'H{\^o}pital's rule is used twice.
See also \cref{figure:modulus} for the illustration of~$K_p^f$.
Thus, \cref{theorem:regret_order} yields~$\limsup_{r\downarrow0}\sigma(r) \ge 2$ and~$S\ge 2$,
which indicates that the polynomial rate of~$\omega^{-1}(\rho)$ cannot be faster than~$\rho^{1/2}$.
The asymptotic lower bound of~$\sigma$ provided here is indeed tight, as we see~$\sigma(r)\to 2$ for~$r\downarrow0$ \cite[Appendix~B.1]{Bao:2023}.
All in all, we asymptotically have the~$1$-norm bound
\[
  |q - \hat q| \lesssim \sqrt{D_{\text{KL}}(q\|\hat q)}
  \quad \text{for $q,\hat q \in [0,1]$,}
\]
recovering Pinsker's inequality.

\paragraph*{Squared norms.}
Consider the squared~$\alpha$-norms for~$\alpha > 1$:
\[
  f(\qbf)=\frac{\|\qbf\|_\alpha^2-1}{2} = \frac12\Biggl(\sum_{n\in[N]}q_n^\alpha\Biggr)^{\frac2\alpha}-\frac12.
\]
By the Savage representation~\eqref{equation:savage}, we can generate proper losses:
\[
  \ell_y(\qbf) = - \|\qbf\|_\alpha^{2-\alpha}q_y^{\alpha-1} + \frac{1+\|\qbf\|_\alpha^2}{2}.
\]
By plugging in~$\alpha=2$, the Brier score in \cref{section:savage} is recovered.

To see the asymptotic speed of the~$p$-norm bound~\eqref{equation:regret_bound_inv} for~$(N,p)=(2,1)$,
we take the limit of~$K_p^f$ similarly to the log-loss case.
By using the closed form of~$\omega(r)$ provided in \cref{table:examples},
when~$\alpha\in(1,2)$,
\[
  \begin{aligned}
    \lim_{r\downarrow0} K_p^f(r)
    &= \lim_{r\downarrow0}\frac{[(1+\frac r2)^\alpha+(1-\frac r2)^\alpha]^{\frac2\alpha}-4^{\frac1\alpha}}{r^2} \\
    &= \frac14\lim_{r\downarrow0}\left\{(2-\alpha)[(1+r)^\alpha+(1-r)^\alpha]^{\frac2\alpha-2}[(1+r)^{\alpha-1}-(1-r)^{\alpha-1}]^2 \right.\\
        &\phantom{2\lim_{r\downarrow0}} \quad \left. + (\alpha-1)[(1+r)^\alpha+(1-r)^\alpha]^{\frac2\alpha-1}[(1+r)^{\alpha-2}+(1-r)^{\alpha-2}]\right\} \\
    &= (\alpha-1)2^{\frac2\alpha-2} < \infty,
  \end{aligned}
\]
where L'H{\^o}pital's rule is used twice.
When~$\alpha\ge 2$, we can similarly show the existence of the limit of~$K_p^f$ as~$r\downarrow0$.
Thus, these losses satisfy both 
\eqref{theorem:condition:modulus_positive} and \eqref{theorem:condition:modulus_conti} of \cref{theorem:regret_order}, which indicates that~$\omega^{-1}(\rho)$ cannot be faster than~$O(\rho^{1/2})$.
Across different~$\alpha>1$, we have~$\sigma(r)\to 2$ as~$r\downarrow0$ \cite[Figure~4]{Bao:2023}, and thus the provided asymptotic lower bound of~$\sigma$ is tight.

\paragraph*{Pseudo-spherical losses.}
Consider the~$\alpha$-norms for~$\alpha>1$:
\[
  f(\qbf) = \frac{\|\qbf\|_\alpha-1}{\alpha-1} = \frac{1}{\alpha-1}\Biggl[\Biggl(\sum_{n\in[N]}q_n^\alpha\Biggr)^{\frac1\alpha} - 1\Biggr],
\]
which has been sometimes used as the~$\alpha$-norm information measure \cite{Boekee1980}.
By the Savage representation~\eqref{equation:savage}, we can generate proper losses:
\[
  \ell_y(\qbf) = \frac{1}{\alpha-1}\Biggl(1-\frac{q_y^{\alpha-1}}{\|\qbf\|_\alpha^{\alpha-1}}\Biggr),
\]
which is called the \emph{pseudo-spherical losses} \cite{Good1971}.
By plugging in~$\alpha=2$, the spherical loss is recovered.
The associated Bregman divergence is
\[
  R(\qbf,\hat\qbf) = \frac{1}{\alpha-1}\Biggl(\|\qbf\|_\alpha - \frac{\inpr{\qbf}{\hat\qbf^{\alpha-1}}}{\|\hat\qbf\|_\alpha^{\alpha-1}}\Biggr) = \frac{1}{\alpha-1}\Biggl(\|\qbf\|_\alpha - \frac{1}{\|\hat\qbf\|_\alpha^{\alpha-1}}\sum_{n\in[N]}q_n\hat q_n^{\alpha-1}\Biggr).
\]
Note that~$\ln[\|\qbf\|_\alpha-(\alpha-1)R(\qbf,\hat\qbf)]/(\alpha-1)$ can be identified with the (cross-entropy of) \emph{gamma-divergences}, which is commonly used in robust regression \cite{Fujisawa2008} and inference with unnormalized models \cite{Kanamori2015}.
With the limit~$\alpha\downarrow1$, the pseudo-spherical loss approaches the log loss, and the associated Bregman divergence approaches to the Kullback--Leibler divergence, correspondingly.

We illustrate the case~$N=2$ and~$\alpha\in(1,3/2]\cup[2,\infty)$ since the modulus~$\omega$ can be written analytically herein (shown in \cref{table:examples}).
When~$1<\alpha\le3/2$, by invoking L'H{\^o}pital's rule twice,
\[
  \begin{aligned}
    \lim_{r\downarrow0}K_p^f(r)
    &= \frac{4}{\alpha-1}\lim_{r\downarrow0}\frac{\left[\left(1+\frac r2\right)^\alpha+\left(1-\frac r2\right)^\alpha\right]^{\frac1\alpha}-2^{\frac1\alpha}}{r^2} \\
    &= \lim_{r\downarrow0}\left\{\frac12\left[\left(1+\frac r2\right)^{\alpha} + \left(1-\frac r2\right)^{\alpha}\right]^{\frac1\alpha-1} \left[\left(1+\frac r2\right)^{\alpha-2} + \left(1-\frac r2\right)^{\alpha-2}\right] \right. \\
      &\phantom{=\lim_{r\downarrow0}} \quad \left. - \frac1\alpha \left[\left(1+\frac r2\right)^{\alpha} + \left(1-\frac r2\right)^{\alpha}\right]^{\frac1\alpha-2} \left[\left(1+\frac r2\right)^{\alpha-1} - \left(1-\frac r2\right)^{\alpha-1}\right] \right\} \\
    &= 2^{\frac1\alpha-1} \\
    &< \infty.
  \end{aligned}
\]
When~$\alpha\ge2$, by invoking L'H{\^o}pital's rule twice,
\[
  \begin{aligned}
    &\lim_{r\downarrow0}K_p^f(r) \\
    &= \frac4{\alpha-1}\lim_{r\downarrow0}\frac{\left[\left(\frac r2\right)^\alpha+\left(1-\frac r2\right)^\alpha\right]^{\frac1\alpha} - \left[\left(\frac r2\right)^\alpha+\left(2-\frac r2\right)^\alpha\right]^{\frac1\alpha} + 1}{r^2} \\
    &= \frac12\lim_{r\downarrow0}\left\{-\left[\left(\frac r2\right)^\alpha+\left(1-\frac r2\right)^\alpha\right]^{\frac{1-2\alpha}{\alpha}}\left[\left(\frac r2\right)^{\alpha-1}-\left(1-\frac r2\right)^{\alpha-1}\right]^2 \right. \\
      &\phantom{= 2\lim_{r\downarrow0}}\quad + \left[\left(\frac r2\right)^\alpha+\left(1- \frac r2\right)^\alpha\right]^{\frac{1-\alpha}{\alpha}}\left[\left(\frac r2\right)^{\alpha-2}+\left(1- \frac r2\right)^{\alpha-2}\right] \\
      &\phantom{= 2\lim_{r\downarrow0}}\quad +\left[\left(\frac r2\right)^\alpha+\left(2-\frac r2\right)^\alpha\right]^{\frac{1-2\alpha}{\alpha}}\left[\left(\frac r2\right)^{\alpha-1}-\left(2- \frac r2\right)^{\alpha-1}\right]^2 \\
      &\phantom{= 2\lim_{r\downarrow0}}\left.\quad - \left[\left(\frac r2\right)^\alpha+\left(2- \frac r2\right)^\alpha\right]^{\frac{1-\alpha}{\alpha}}\left[\left(\frac r2\right)^{\alpha-2}+\left(2-\frac r2\right)^{\alpha-2}\right] \right\} \\
    &= \begin{cases}
      \frac14 & \text{if $\alpha = 2$} \\
      0 & \text{if $\alpha > 2$}
    \end{cases} \\
    &< \infty.
  \end{aligned}
\]
Thus, these losses satisfy \eqref{theorem:condition:modulus_conti} of \cref{theorem:regret_order}, which indicates that~$\omega^{-1}(\rho)$ cannot be faster than~$O(\rho^{1/2})$.
Compared with the log loss and the squared norms, the pseudo-spherical losses are more interesting examples for us because~$K_p^f(r)$ is no longer always positive.
Indeed, \eqref{theorem:condition:modulus_positive} of \cref{theorem:regret_order} is satisfied when~$\alpha\in(1,3/2]\cup\set{2}$ but not satisfied when~$\alpha>2$.
The previous $O(\rho^{1/2})$ lower bounds typically require the local strong convexity \cite{Frongillo:2021} \cite{Mao2024NeurIPS}, which is similar to \eqref{theorem:condition:modulus_conti}.
Therefore, our \cref{theorem:regret_order} slightly lifted the assumptions, requiring \eqref{theorem:condition:modulus_conti} solely.
See \cref{figure:modulus} to confirm that~$K_p^f$ is indeed asymptotically vanishing with the case~$\alpha=2.5$, for which \eqref{theorem:condition:modulus_positive} no longer holds.

\paragraph*{Tsallis losses.}
Consider the negative Tsallis~$\alpha$-entropy
\[
  f(\qbf) = \frac{\|\qbf\|_\alpha^\alpha-1}{\alpha-1} = \frac{1}{\alpha-1}\Biggl(\sum_{n\in[N]}q_n^\alpha - 1\Biggr)
\]
as a convex potential, for~$\alpha>1$.
The Tsallis entropies generalize the Shannon entropy for non-extensive systems \cite{Tsallis1988}, and recovers the Shannon entropy at the limit~$\alpha\downarrow1$.
By the Savage representation~\eqref{equation:savage}, we can generate proper losses:
\[
  \ell_y(\qbf) = -\frac{\alpha q_y^{\alpha-1}-1}{\alpha-1} + \|\qbf\|_\alpha^\alpha,
\]
which recovers the log loss at the limit~$\alpha\downarrow1$.
We call them the~\emph{$\alpha$-log loss} for convenience.
Note that this loss is slightly different from the~$\alpha$-loss~$\ell_y(\qbf) = -\alpha(q_y^{1-1/\alpha}-1)/(\alpha-1)$ \cite{Sypherd2019ISIT}.
Indeed, the~$\alpha$-log loss is proper by its construction, while the~$\alpha$-loss is known to be improper \cite{Sypherd2022ICML}; despite that both of them approach the log loss at the same limit.
The associated Bregman divergence to the~$\alpha$-log loss is
\[
  R(\qbf,\hat\qbf) = \frac{\|\qbf\|_\alpha^\alpha - \alpha\inpr{\qbf}{\hat\qbf^{\alpha-1}} + (\alpha-1)\|\hat\qbf\|_\alpha^\alpha}{\alpha-1},
\]
which is the \emph{Tsallis divergence} \cite{Dawid2007AISM}, and also corresponds to \emph{density power divergence} (or the beta-divergence) \cite{Basu1998} up to constant, used in robust statistics.
The Tsallis divergence interpolates the Kullback--Leibler divergence and the squared~$2$-norm distance at the limits of~$\alpha\to1$ and~$\alpha\to2$, respectively.
Some literature opts another definition of the Tsallis divergence, defined by replacing~$\ln$ in the Kullback--Leibler divergence with the~$\alpha$-logarithmic function \cite{Amid2019AISTATS}---precisely, this another definition should be distinguished as the~$t$-divergence \cite{Ding2011NeurIPS}.

To see the asymptotic speed of the~$p$-norm bound~\eqref{equation:regret_bound_inv} for~$(N,p)=(2,1)$, we take the limit of~$K_p^f$.
When~$\alpha\in(1,2)\cap(3,\infty)$, the explicit form of the modulus~$\omega$ in \cref{table:examples} yields
\[
  \begin{aligned}
    \lim_{r\downarrow0}K_p^f(r)
    &= \frac{8}{2^\alpha(\alpha-1)}\lim_{r\downarrow0}\frac{(1+\frac r2)^\alpha+(1-\frac r2)^\alpha-2}{r^2} \\
    &= \frac{2\alpha}{2^\alpha}\lim_{r\downarrow0}\frac{(1+\frac r2)^{\alpha-2}+(1-\frac r2)^{\alpha-2}}{2} \\
    &= \alpha2^{1-\alpha} < \infty,
  \end{aligned}
\]
where the L'H{\^o}pital's rule is invoked twice.
When~$2\le\alpha\le3$,
\[
  \begin{aligned}
    \lim_{r\downarrow0}K_p^f(r)
    &= \frac{4}{\alpha-1}\lim_{r\downarrow0}\frac{\left[\left(\frac r2\right)^\alpha+\left(1-\frac r2\right)^\alpha\right]-2^{1-\alpha}\left[\left(\frac r2\right)^\alpha+\left(2-\frac r2\right)^\alpha\right]+1}{r^2} \\
    &= \frac12\alpha\lim_{r\downarrow0}\left\{\left[\left(\frac r2\right)^{\alpha-2}+\left(1-\frac r2\right)^{\alpha-2}\right]-2^{1-\alpha}\left[\left(\frac r2\right)^{\alpha-2}+\left(2-\frac r2\right)^{\alpha-2}\right]\right\} \\
    &= \frac\alpha 4 < \infty,
  \end{aligned}
\]
where the L'H{\^o}pital's rule is invoked twice.
In either case, these losses satisfy both \eqref{theorem:condition:modulus_positive} and \eqref{theorem:condition:modulus_conti} of \cref{theorem:regret_order}, which indicates that~$\omega^{-1}(\rho)$ cannot be faster than~$O(\rho^{1/2})$.

\paragraph*{Non-differentiable generator.}
While all the above examples are generated by differentiable convex generator~$f$, our \cref{theorem:regret_order} is applicable even to non-differentiable~$f$, relaxing the previous lower bounds on surrogate regret bounds \cite{Frongillo:2021} \cite{Mao2024NeurIPS}.
To demonstrate its full capacity, we artificially consider the following non-differentiable convex function:
\[
  f(\qbf) = \max_{n\in[N]}\left(q_n-\frac23\right)^2-\frac19 = \left(\min_{n\in[N]}q_n-\frac23\right)^2 - \frac49,
\]
which reduces to
\[
  f(q) = \begin{cases}
    \left(\frac23-q\right)^2 - \frac49 & \text{if $q \in [0,\frac12]$} \\
    \left(q-\frac13\right)^2 - \frac49 & \text{if $q \in (\frac12,1]$}
  \end{cases}
\]
for~$(N,p)=(2,1)$.
This is non-differentiable at~$q=1/2$.
For~$(N,p)=(2,1)$, we can generate the binary loss by Savage representation~\eqref{equation:savage} as follows:
for~$y=1$,
\[
  \ell(q) = \begin{cases}
    q^2 - 2q + \frac43 & \text{if $q \in [0,\frac12]$} \\
    q^2 - 2q + 1 & \text{if $q \in (\frac12,1]$}
  \end{cases}
  ,
\]
and~$\ell(1-q)$ for~$y=2$,
which are discontinuous at~$q=1/2$ yet strictly proper due to the strict convexity of~$f$.
For this example,~$\omega(r)=r^2/4$ holds for~$r\in[0,1/2]$, and hence
\[
  \lim_{r\downarrow0}K_p^f(r) = \lim_{r\downarrow0} \frac{8\omega(r)}{r^2} = 2 < \infty.
\]
Thus, \eqref{theorem:condition:modulus_conti} is satisfied.

\paragraph*{When $N \geq 3$.}
Though deriving a closed form of~$\omega$ for general~$N>2$ is challenging,
we can delineate~$\omega$ for the negative Shannon entropy with~$p=2$:
for~$r \in (0,2^{1/2})$, define $\qbf, \check\qbf \in \simplex$ by
\[
  \qbf = \begin{bmatrix}
    \dfrac{1+2^{-1/2}r}{2} & \dfrac{1-2^{-1/2}r}{2} & 0 & \!\!\dots\!\! & 0
  \end{bmatrix}^\top \text{~and~}
  \check\qbf = \begin{bmatrix}
    \dfrac{1-2^{-1/2}r}{2} & \dfrac{1+2^{-1/2}r}{2} & 0 & \!\!\dots\!\! & 0
  \end{bmatrix}^\top.
\]
Then, $\|\qbf - \check\qbf\|_2 = r$ and $\omega(r) = J(\qbf, \check\qbf)$.
At this minimizer, $\omega$ can be written as
\[
  \omega(r) = \frac{1+2^{-1/2}r}{2}\ln\frac{1+2^{-1/2}r}{2} + \frac{1-2^{-1/2}r}{2}\ln\frac{1-2^{-1/2}r}{2} + \ln2.
\]
This $\omega$ (for $p=2$) is akin to the form of $\omega$ shown in \cref{table:examples}, which is for $(N,p)=(2,1)$, with a slight difference in the scale.
Its derivation is based on the method of Lagrange multipliers and deferred to \cref{proposition:shannon_general_minimizer}, which is highly non-trivial and interesting in its own right.

To apply \cref{theorem:regret_order} for this example, let us confirm \eqref{theorem:condition:modulus_conti} is satisfied by taking the asymptotic limit of~$K_p^f(r)$.
By invoking L'H{\^o}pital's rule twice, we have
\[
  \begin{aligned}
    \lim_{r\downarrow0}K_p^f(r)
    &= 8\lim_{r\downarrow0}\frac{\frac{1+2^{-1/2}r}{2}\ln\frac{1+2^{-1/2}r}{2} + \frac{1-2^{-1/2}r}{2}\ln\frac{1-2^{-1/2}r}{2} + \ln2}{r^2} \\
    &= 2\lim_{r\downarrow0}\frac{1}{(1+2^{-1/2}r)(1-2^{-1/2}r)} \\
    &= 2 < \infty,
  \end{aligned}
\]
which indicates \eqref{theorem:condition:modulus_conti} is satisfied.
Thus,~$\omega^{-1}(\rho)$ cannot be faster than~$O(\rho^{1/2})$.

\section{Conclusion}
\label{section:conclusion}
In this work, we examine surrogate regret bounds on the~$p$-norm,~$\|\qbf-\hat\qbf\|_p\lesssim\omega^{-1}(R(\qbf,\hat\qbf))$, which measures predictive performances of plug-in forecasters under downstream tasks such as classification and ranking.
A surrogate regret bound is characterized by the modulus of convexity~$\omega$ associated with the Bregman generator~$f=-\underline{L}$ for a given proper loss.
First, we show that the existence of non-vacuous regret bounds is equivalent to the strict properness of losses.
Then, we prove that the~$p$-norm upper bound~$\omega^{-1}(\rho)$ cannot be faster than the~$1/2$-order of surrogate regrets~$O(\sqrt\rho)$ for a wide range of strictly proper losses.
Herein, the assumptions on loss functions are greatly relaxed so that we do not require the differentiability or the local strong convexity of loss functions anymore.
We demonstrate that many loss functions such as the log loss, Brier loss, pseudo-spherical losses, and~$\alpha$-log losses satisfy the assumptions of our optimal-order argument.

As a side note, there is a fundamental relationship between a proper loss and a convex body \cite{Williamson2014COLT} \cite{Williamson2023JMLR}.
Specifically, the Bayes risk~$\underline{L}$ of a proper loss is the support function of the superprediction set (which is a convex body in~$\Rbb^N$) associated with the proper loss.
Hence, we can work on a convex body instead of directly working on a proper loss.
This perspective has been used to consider aggregating algorithms.
In this connection, we studied the modulus of convexity of Bregman generators~$f=-\underline{L}$, while the modulus of convexity of Banach spaces has been more commonly studied to measure the set curvature \cite{Figiel:1976}.
We conjecture that the modulus of convexity of Bregman generators has a tight connection to the modulus of convexity of superprediction sets, which remains an interesting open question from the viewpoint of convex analysis.

This work is concerned with only \emph{expected} surrogate regrets, induced from the full risk~$\Lbb[\hat\qbf]$ in \cref{section:proper_loss_sub}.
Yet, its empirical estimation and optimization together play an important role in realistic learning scenarios.
To take them into account, it is more standard to consider a learner acting on~$\Rbb^N$-valued margin, instead of~$\simplex$-valued prediction as supposed in \cref{section:proper_losses}.
Proper \emph{composite} losses are common therein, where a link function connects a probabilistic report~$\hat\qbf\in\simplex$ to a real-valued report on~$\Rbb^N$.
Hence, we can consider estimation and optimization of~$\Rbb^N$-valued functions.
In spite of the scarcity, the estimation error rates of class probability models under the binary case~$N=2$ have been studied rigorously for linear hypotheses \cite{Telgarsky2015COLT},
while the optimization error rates of proper composite losses by gradient descent under the binary case have been characterized recently \cite{Bao2025}.
We hope to thoroughly understand how a proper loss behaves by integrating surrogate regrets, estimation, and optimization all at once, and leave this for future work.

\section*{Acknowledgments}
HB and AT are supported by JSPS Grant-in-Aid for Transformative Research Areas(A) (22A201).
AT is supported by JSPS Grant-in-Aid for Scientific Research(C)
(19K03494).

\bibliography{reference}

\appendix

\section{Subgradient inequality}
\label{section:subgradient_inequality}
In this paper, we adopted a slightly non-conventional definition of subdifferentials $\partial f$ in \cref{section:background}
to allow some elements of the subgradient to be~$-\infty$.
\begin{restatable}{lemma}{subdifferential}
  \label{lemma:subdifferential}
  Let~$f:\Rbb^N \to (-\infty, \infty]$ be a proper convex function such that~$\simplex \subset \domain{f}$.
  For~$\qbf^0 \in \simplex$, the set~$\partial f(\qbf^0)$ is nonempty.
\end{restatable}
\begin{proof}
  Fix~$\qbf^0 \in \simplex$.
  If~$f$ is subdifferentiable at~$\qbf^0$,
  then its subgradient belongs to $\partial f(\qbf^0)$ and the claim holds true.
 Thus we assume that~$f$ is not subdifferentiable at~$\qbf^0$.
  Since~$f$ is subdifferentiable at~$\qbf^0 \in \simplex_+$ \cite[Theorem~23.4]{Rockafeller:1970},~$I \defeq |\supp(\qbf^0)|$ satisfies~$1 \le I \le N-1$.
  For~$\etabf \in \Rbb^{I}$, define~$\xibf^{\etabf}\in \Rbb^N$ by 
  \[
    \xi^{\etabf}_n \defeq \begin{cases}
      \eta_n & \text{if } n \in \supp(\qbf^0), \\
      0 & \text{if } n \notin \supp(\qbf^0),
    \end{cases}
  \]
  and define a function~$f_I: \Rbb^I \to (-\infty, \infty]^N$ by
  \[
    f_I(\etabf) \defeq f(\xibf^{\etabf})
    \quad \text{for } \etabf \in \Rbb^I.
  \]
  Then,~$f_I$ is a proper convex function on~$\Rbb^I$ such that~$\triangle^I \subset \domain{f}_{I}$,
  consequently,~$f_I$ is subdifferentiable at~$\hat\etabf \in \triangle^I_+$ \cite[Theorem~23.4]{Rockafeller:1970}.
  For~$\qbf \in \simplex$, define~$\etabf^{\qbf}\in \Rbb^I$ by~$\eta^{\qbf}_n = q_n$ for~$n \in \supp(\qbf^0)$.
  Then, for~$\qbf \in \simplex$ with~$\supp(\qbf) \subset \supp(\qbf^0)$,
  we have~$\etabf^{\qbf} \in \triangle^I$ and~$f_I(\etabf^{\qbf}) = f(\qbf)$.
  Moreover, if~$\qbf \in \simplex$ satisfies~$\supp(\qbf) = \supp(\qbf^0)$,
  then~$\etabf^{\qbf} \in \triangle_+^I$ and hence~$\partial f_I(\etabf^{\qbf}) \neq \emptyset$.
  Choose~$\wbf \in \partial f_I(\etabf^{\qbf^0})$ and define~$\vbf \in [-\infty,\infty)^N$ by
  \[
    {v}_n \defeq \begin{cases}
      {w}_n & \text{if } n \in \supp(\qbf^0), \\
      -\infty & \text{if } n \notin \supp(\qbf^0).
    \end{cases}
  \]
  From now on, we show that $\vbf \in \partial f(\qbf^0)$.
  For~$\qbf \in \simplex$, if~$q_n > 0$ holds for some~$n \notin \supp(\qbf^0)$, then~$\inpr{\vbf}{\qbf - \qbf^0} = -\infty$ and \eqref{equation:subtangent_inequality} holds.
  On the other hand, if~$q_n = 0$ for all~$n \notin \supp(\qbf^0)$, then~$\qbf \in \simplex$ with~$\supp(\qbf) \subset \supp(\qbf^0)$ and 
  \[
    f(\qbf)
    = f_I(\etabf^\qbf)
    \geq f_I\Big(\etabf^{\qbf^0}\Big) + \inpr{\wbf}{\etabf^\qbf - \etabf^{\qbf^0}}
    = f(\qbf^0) + \inpr{\vbf}{\qbf - \qbf^0},%
  \]
  that is, \eqref{equation:subtangent_inequality} holds for~$\qbf \in \simplex$.
  Thus, $\vbf \in \partial f(\qbf^0)$ follows.
  This completes the proof of the lemma.
\end{proof}

\section{An example of empty \texorpdfstring{$\Mcal(\qbf)$}{M(q)}}
\label{section:empty_minimizer_set}
If a loss $\ellbf$ is not lower semi-continuous (as in \cref{lemma:minimizability}), the set of its minimizers $\Mcal(\qbf)$ (introduced in \cref{section:proper_loss_sub}) can be empty.

Consider the following example:
\[
\ell_y(\hat\qbf) = \begin{cases}
  1 - \hat q_y & \text{if $\hat q_y\ne 1$,} \\
  1 & \text{if $\hat q_y=1$.}
\end{cases}
\]
Then,
\[
  L(\ebf_1, \hat\qbf) = \ell_1(\hat\qbf) = \begin{cases}
    1 - \hat q_1 & \text{if $\hat q_1\ne 1$,} \\
    1 & \text{if $\hat q_1=1$,}
  \end{cases}
\]
where $\ebf_1\defeq[1,0,\dots,0]^\top\in\Rbb^N$.
In this case, we have
\[
  \inf_{\hat\qbf \in \simplex} L(\ebf_1, \hat\qbf) = 0,
\]
but there does not exist $\hat\qbf \in \simplex$ such that $L(\ebf_1, \hat\qbf)=0$.
Thus, $\Mcal(\ebf_1)=\emptyset$.

\section{Proof of power evaluations without differentiability}
\label{section:proof:power_evaluation}

In \cref{section:regret_order}, we show power evaluations of the moduli with the differentiability of~$\omega$.

\polynomialrate*

\begin{proof}
We first show that $r \mapsto \omega(r)r^{-S}$ is non-increasing on $(0,r_0)$ in a similar way to the existing argument \cite[Lemma~2.9]{Ohta:2013}.
For $r \in (0, r_0]$ and $\delta > 0$, there exists $\epsilon_{r,\delta} > 0$ such that
   \begin{equation}\label{equation:dini_deriv_bound}
     \frac{r}{\omega(r)} \cdot \sup_{\epsilon \in (0, \epsilon_{r,\delta})} \frac{\omega(r) - \omega(r-\epsilon)}{\epsilon}
     \le
   \sigma(r)+ \frac12\delta  
   \end{equation}
by the definition of $D^-\omega$.
   Define
   \[
     g(t) \defeq S + \frac12\delta + \frac{1}{t}[(1-t)^{S+\delta} - 1]
     \quad \text{for $t \in (0,1)$.}
   \]
   Then, $g$ is continuous on $(0,1)$ and
   \[
     \lim_{t \downarrow 0} g(t) = S + \frac12\delta - (S + \delta) =- \frac12\delta < 0,
   \]
   which implies the existence of~$\tau \in (0,1)$ such that~$g(t) < 0$ for~$t \in (0,\tau)$.
   Then, for any~$u \in (0,\epsilon_{r,\delta}) \cap (0,r\tau)$,
   \[
     \begin{aligned}
       -\frac{r^{S+\delta}}{\omega(r)} + \frac{(r-u)^{S+\delta}}{\omega(r-u)}
       &= \frac{ur^{S+\delta-1}}{\omega(r-u)} \left[\frac{r}{\omega(r)} \frac{\omega(r) - \omega(r-u)}{u}\right] - \frac{r^{S+\delta}}{\omega(r-u)} + \frac{r^{S+\delta}(1-\frac{u}{r})^{S+\delta}}{\omega(r-u)} \\
      &\le \frac{ur^{S+\delta-1}}{\omega(r-u)}\left[\sigma(r) + \frac12\delta\right] + \frac{ur^{S+\delta-1}}{\omega(r-u)} \cdot \frac{1}{\frac{u}{r}} \left[\left(1-\frac{u}{r}\right)^{S+\delta} - 1\right] \\
       &\le \frac{ur^{S+\delta-1}}{\omega(r-u)} \left[S+\frac12\delta + g\left(\frac{u}{r}\right) - \left(S + \frac12\delta\right)\right] \\
       &= \frac{ur^{S+\delta-1}}{\omega(r-u)} \cdot g\left(\frac{u}{r}\right) \\
      &< 0,
     \end{aligned}
   \]
   where the inequality \eqref{equation:dini_deriv_bound} is used at the second line
   and the third line follows from the definition of~$S$.
   Hence, we have
   \[
     \frac{(r-u)^{S+\delta}}{\omega(r-u)} < \frac{r^{S+\delta}}{\omega(r)}  
     \quad \text{for $r \in (0,r_0)$ and $u \in (0,\epsilon_{r,\delta}) \cap (0, r\tau)$.}
   \]
Letting~$\delta \downarrow 0$, we conclude that~$r \mapsto r^S/\omega(r)$ is non-decreasing on~$(0,r_0)$.
This is equivalent to that~$r \mapsto \omega(r)r^{-S}$ is non-increasing on~$(0,r_0)$.

Next, we show that~$r \mapsto \omega(r)r^{-s}$ is non-decreasing on~$(0,r_0)$.
For~$r \in (0,r_0]$, let~$u \in (0, r)$.
Since~$\ln\omega$ is non-decreasing on~$[r-u,r]$,
it follows from the fundamental theorem of calculus \cite[Theorem 1.3.1]{KKK} that 
\[
\int_{r-u}^{r} D^- \ln \omega(r') \rd{r'} \leq \ln \frac{\omega(r)}{\omega(r-u)}.
\]
By~$D^-\omega(r')\leq S<\infty$, we have
\[
\lim_{\epsilon\downarrow 0} \omega(r'-\epsilon)=\omega(r')
\quad \text{for $r'\in (0,r_0)$,}
\]
which yields
\[
\lim_{\epsilon \downarrow 0}
\frac{\ln\omega(r')-\ln\omega(r'-\epsilon)}{\omega(r')-\omega(r'-\epsilon)}
=\frac{1}{\omega(r')}
\]
and
\[
D^-\ln \omega(r')
=
 \limsup_{\epsilon \downarrow 0}
\left[ \frac{\ln\omega(r')-\ln\omega(r'-\epsilon)}{\omega(r')-\omega(r'-\epsilon)}
\cdot\frac{\omega(r')-\omega(r'-\epsilon)}{\epsilon} \right]
=\frac{1}{\omega(r')} D^-\omega(r') \geq \frac{s}{r'}.
\]
Thus, we have
\[
\int_{r-u}^{r} D^-\ln \omega(r') \rd{r'} \geq 
\int_{r-u}^{r} \frac{s}{r'}\rd{r'}
=s\ln\frac{r}{r-u}.
\]
These imply 
\[
\omega(r-u)\cdot(r-u)^{-s} \leq \omega(r)\cdot r^{-s},
\]
that is,~$r \mapsto \omega(r)r^{-s}$ is non-decreasing on~$(0,r_0)$.
\end{proof}

\section{Continuity property of local modulus of convexity}
\label{section:property_local_modulus}

The local modulus of convexity defined in \cref{definition:local_modulus} naturally entails the following continuity property.
We state it in the following lemma for the sake of completeness.
\begin{restatable}{lemma}{modulusconti}
  \label{lemma:modulus_local_conti}
  If~$f:\simplex \to \Rbb$ is continuous convex,
  then~$K^{f}_p:(0,2^{1/p}] \to \Rbb$ is lower semi-continuous and left-continuous.
\end{restatable}
To prove \cref{lemma:modulus_local_conti}, we use the following supplement result.
\begin{lemma}
  \label{lemma:supp:modulus_bound}
  If~$f:\simplex\to\Rbb$ is continuous convex, then for~$r\in(0,2^{1/p}]$ and~$\tau\in(0,1/2)$, we have
  \[
    \frac{K_p^f((1-2\tau)r)}{8}[(1-2\tau)r]^2
    = \omega((1-2\tau)r)
    \le \omega(r) - \frac{\kappa_p^f}{2}\tau(1-\tau)r^2
    \le \frac{K_p^f(r)}{8}r^2 - \frac{\kappa_p^f}{2}\tau(1-\tau)r^2.
  \]
\end{lemma}
\begin{proof}
  Choose distinct~$\qbf, \check\qbf \in \simplex$ satisfying
  \[
    \|\qbf - \check\qbf\|_p = r
    \quad \text{and} \quad \omega(r) = J(\qbf, \check\qbf),
  \]
  which exist thanks to \cref{lemma:moduli_minimizer} together with the convexity of~$f$.
  For these~$\qbf, \check\qbf$, define
  \[
    c(t) \defeq (1-t)\qbf + t\check\qbf
    \quad \text{for $t \in [0,1]$.}
  \]
  For~$\tau \in (0,1/2)$, we have
  \[
    \begin{aligned}
      \frac{K^{f}_p((1-2\tau)r)}{8}[(1-2\tau)r]^2
      &= \omega((1-2\tau)r) \\
      &\leq J(c(\tau), c(1-\tau)) \\
      &\leq J(c(0), c(1)) - \frac{\kappa^f_p}{2}\tau(1-\tau)r^2 \\
      &=\omega(r)-\frac{\kappa^f_p}{2}\tau(1-\tau)r^2\\
      &\leq \frac{K^f_p(r)}{8}r^2-\frac{\kappa^f_p}{2}\tau(1-\tau)r^2.
    \end{aligned}
  \]
\end{proof}
{\renewenvironment{proof}{\par\noindent{\bf Proof of \cref{lemma:modulus_local_conti}\ }}{\hfill$\openbox$}
\begin{proof}
Fix~$r \in (0,2^{1/p}]$.
Let~$(r_j)_{j \in \Nbb} \subset (0,2^{1/p}]$ be a sequence converging to~$r$.
For each~$j \in \Nbb$, there exist~$\qbf^j, \check\qbf^j \in \simplex$ satisfying
\[
  \|\qbf^j - \check\qbf^j\|_p = r_j \quad \text{and} \quad \omega(r_j) = J(\qbf^j, \check\qbf^j)
\]
from \cref{lemma:moduli_minimizer}.
Define~$c_j:[0,1] \to \simplex$ by
\[
  c_j(t) \defeq (1-t)\qbf^j + t\check\qbf^j \quad \text{for $t \in [0,1]$.}
\]
By the Arzel\'a--Ascoli theorem, we can extract a subsequence~$(c_{j_m})_{m\in \Nbb}$ converging uniformly to some~$c:[0,1] \to \simplex$ uniformly, where
\[
  c(t) = (1-t)c(0) + tc(1) \quad \text{for $t \in [0,1]$~~and} \quad \|c(0) - c(1)\|_p = r
\]
hold.
Since~$f$ is continuous, we have
\begin{align*}
  K^{f}_p(r)
  = \frac{8}{r^2}\omega(r)
  &\leq \frac{8}{r^2} J(c(0),c(1))\\
  &= \lim_{m \to \infty} \frac{8}{r_{j_m}^2} J(c_{j_m}(0),c_{j_m}(1))
  = \lim_{m \to \infty} \frac{8}{r_{j_m}^2} \omega(r_{j_m})
  = \lim_{m \to \infty} K^{f}_p(r_{j_m}).
\end{align*}
Thus,~$K^{f}_p:(0,2^{1/p}] \to \Rbb$ is lower semi-continuous.

Next, by \cref{lemma:supp:modulus_bound}, we have
\[
  \frac{K_p^f((1-2\tau)r)}{8}[(1-2\tau)r]^2
  \le \frac{K_p^f(r)}{8}r^2 - \frac{\kappa_p^f}{2}\tau(1-\tau)r^2
\]
for $\tau\in(0,1/2)$.
Dividing by~$[(1-2\tau)r]^2/8 \neq 0$ and then taking the limit yields 
\[
  \limsup_{\tau \downarrow 0} K^{f}_p((1-2\tau)r)
  \leq K^{f}_p(r).
\]
Together with the lower semi-continuity~$K_p^f$, the left-continuity~$K_p^f$ is ensured.

This completes the proof of the lemma.
\end{proof}
}

\section{Deferred proofs}
\label{section:proofs}

\minimizability*

\begin{proof}
  For~$\qbf \in \simplex$, the lower semi-continuity of~$L(\qbf,\cdot)$ follows from that of~$\ellbf$,
  which ensures the closedness of~$\Mcal(\qbf)$.
  Moreover, by the extreme value theorem with the closedness of~$\simplex$,
  we see~$\Mcal(\qbf)\neq \emptyset$ holds.

  Assume the continuity of~$\ellbf$ 
  and we show the existence of a Borel selector of~$\Mcal$.
  Note that the continuity of~$\ell$ guarantees 
  the continuity of~$L(\qbf, \cdot)$
  on~$\simplex$ for each~$\qbf\in \simplex$.
  By the Kuratowski and Ryll-Nardzewski measurable selection theorem~\cite[Main Theorem \& Corollary 1]{KuratowskiRyll-Nardzewski65},
  it is enough to show that
  \[
    \Bcal_\Kcal \defeq \setcomp{\qbf \in \simplex}{\Mcal(\qbf) \cap \Kcal \neq \emptyset}
  \]
  is Borel for any %
  closed set~$\Kcal$ in~$\Rbb^N$ with~$\simplex \cap \Kcal \neq \emptyset$.

  Fix a compact set~$\Kcal$ in~$\Rbb^N$ with~$\simplex \cap \Kcal \neq \emptyset$.
  Since 
  $\simplex \cap \Kcal$ is a compact metric space hence separable, there exists a dense countable set~$\set{\qbf^j}_{j \in \Nbb}$ in~$\simplex \cap \Kcal$.
  For each~$j\in \Nbb$, define~$d_j: \simplex \to \Rbb$ by
  \[
    d_j(\qbf) \defeq L(\qbf, \qbf^j) - \inf_{\hat\qbf \in \simplex} L(\qbf, \hat\qbf)
    \quad \text{for } \qbf \in \simplex,
  \]
  which is lower semi-continuous, in particular, Borel on~$\simplex$.
  Then,
  \[
    \Bcal \defeq \bigcap_{m \in \Nbb} \bigcup_{j \in \Nbb} d_j^{-1}([0, m^{-1}))
  \]
  is Borel.
  We will show~$\Bcal = \Bcal_\Kcal$.
  For~$\qbf \in \Bcal_\Kcal$, there exists~$\hat\qbf \in \Mcal(\qbf) \cap \Kcal$.
  By the continuity of~$L(\qbf, \cdot)$ on~$\simplex$, for each~$m \in \Nbb$,
  there  exists~$\delta_m > 0$ such that if~$\qbf' \in \simplex$ satisfies~$\|\qbf' - \hat\qbf\|_2 < \delta_m$,
  then~$0 \leq L(\qbf, \qbf') - L(\qbf, \hat\qbf) < m^{-1}$.
  By the density of~$\set{\qbf^j}_{j \in \Nbb}$, there exists~$j_m \in \Nbb$ such that~$\|\qbf^{j_m} - \hat\qbf\|_2 < \delta_m$ and hence
  \[
    d_{j_m}(\qbf) = L(\qbf, \qbf^{j_m}) - L(\qbf, \hat\qbf) \in [0, m^{-1}),
  \]
  which in turn implies~$\qbf \in \Bcal$.
  Conversely, for~$\qbf \in \Bcal$ and~$m \in \Nbb$, there exists~$j_m \in \Nbb$ such that 
  \[
    L(\qbf, \qbf^{j_m}) < \inf_{\hat\qbf \in \simplex} L(\qbf, \hat\qbf) + m^{-1}.
  \]
  We extract a convergent subsequence of~$(\qbf^{j_m})_{m\in \Nbb}$ (not relabeled) with limit~$\hat\qbf \in \simplex\cap \Kcal$.
  The continuity of~$L(\qbf, \cdot)$ gives
  \[
    L(\qbf, \hat\qbf)
    =\lim_{m \to \infty} L(\qbf, \qbf^{j_m})
    \leq \inf_{\hat\qbf' \in \simplex} L(\qbf, \hat\qbf'),
  \]
  proving~$\hat\qbf \in \Mcal(\qbf)$ hence~$\qbf \in \Bcal_\Kcal$.
  This completes the proof of the lemma.
\end{proof}

\jensengapaffine*

\begin{proof}
  We only show that~$f - g$ is affine under the assumption that the midpoint Jensen gaps of~$f$ and~$g$ are the same since the converse implication is trivial.
  Hereafter, let us write the midpoint Jensen gaps of~$f$ and~$g$ by~$J_f$ and~$J_g$, respectively.

  Without loss of generality, we pick~$\qbf^0 \in \simplex_+$ such that~$f$ and~$g$ are differentiable at~$\qbf^0$ because a convex function is differentiable almost everywhere in the interior of its domain.
  Define
  \[
    \vbf^0 \defeq \nabla f(\qbf^0) - \nabla g(\qbf^0)
    \quad \text{and} \quad
    \lambda \defeq f(\qbf^0) - g(\qbf^0).
  \]
  Fix any~$\qbf \in \simplex$ and set 
  \[
    h(t) \defeq f(\qbf^0 + t(\qbf - \qbf^0)) - g(\qbf^0 + t(\qbf - \qbf^0)) - \inpr{\vbf^0}{t(\qbf - \qbf^0)} - \lambda
    \quad \text{for $t \in [0,1]$.}
  \]
  Then,~$h: [0,1] \to \Rbb$ is continuous with~$h(0) = 0$ and~$h'(0) = 0$.
  With elementary algebra, we have
  \[
    0 = J_f(\qbf^0, \qbf^0 + t(\qbf - \qbf^0)) - J_g(\qbf^0, \qbf^0 + t(\qbf - \qbf^0))
    = \frac12h(t) - h\left(\frac{t}{2}\right)
    \quad \text{for all $t \in [0,1]$,}
  \]
  which implies
  \[
    h\left(\frac12t\right) = \frac12h(t)
    \quad \text{for all $t \in [0,1]$.}
  \]
  By invoking this relation recursively, we have~$h(t) = 2^{k}h(2^{-k}t)$ for any~$k \in \Nbb$, which yields
  \[
    h(1) = \lim_{k \to \infty}\frac{h(2^{-k}) - h(0)}{2^{-k}}
    = h'(0) = 0.
  \]
  Consequently, we have
  \[
    f(\qbf) = g(\qbf) + \inpr{\vbf^0}{\qbf - \qbf^0} + \lambda.
  \]
  Thus, we have shown that~$f-g$ is affine.
\end{proof}

\minimizer*

\begin{proof}
  Let~$r \in [0, 2^{1/p}]$.
  Define
  \[
    \Dcal^N(r) \defeq \setcomp{(\qbf,\check\qbf) \in \simplex \times \simplex}{\|\qbf-\hat\qbf\|_p \geq r}.
  \]
  Since~$\Dcal^N(r)$ is compact and~$J$ is continuous on~$\Dcal^N(r)$, there is~$(\qbf, \check\qbf) \in \Dcal^N(r)$ such that~$\omega(r) = J(\qbf, \check\qbf)$.
  Define~$c: [0,1] \to \simplex$ by 
  \[
    c(t) \defeq (1-t)\qbf + t\check\qbf \quad \text{for $t \in [0,1]$.}
  \]
  In the case of~$\|\qbf-\check\qbf\|_p= r$, we can take~$(\qbf^r,\check\qbf^r) = (\qbf,\check\qbf)$, and the statement follows.
  Assume~$\|\qbf-\check\qbf\|_p > r$.
  Then, there exists~$\tau\in (0,1/2]$ such that
  \[
    \|c(\tau)-c(1-\tau)\|_p =(1-2\tau)\|\qbf-\check\qbf\|_p =r.
  \]
  Since~$f\circ c: [0,1]\to \Rbb$ is convex, we have  
  \begin{align}
    \label{properineq}
    \frac{f(c(\tau)) - f(c(0))}{\tau}
    \leq \frac{f(c(1)) -f(c(1-\tau))}{\tau}, 
    \end{align}
  which is equivalent to 
  \begin{align*}
    J(c(\tau),c(1-\tau))
    &=\frac{f(c(\tau))+f(c(1-\tau)}{2} -f(c(1/2)) \\
    &\leq 
    \frac{f(c(0))+f(c(1)}{2} -f(c(1/2)) 
    =J(\qbf, \check\qbf).
  \end{align*}
  This yields~$J(c(\tau),c(1-\tau)) = \omega(r)$, and hence, we can take~$(\qbf^r,\check\qbf^r) = (c(\tau),c(1-\tau))$.
  Thus, we have confirmed the statement.
\end{proof}

\modulusmidpoint*

\begin{proof}
  By definition,~$\kappa^f_p \leq \kappa^{f,1/2}_p$ trivially holds.
  We shall prove the converse inequality.
  Observe from the definition that
  \[
    f\left(\frac{\qbf+\check\qbf}{2}\right) \leq \frac12 f(\qbf) +\frac12f(\check\qbf) - \frac{\kappa^{f,\frac12}_p}{8} \|\qbf-\check\qbf\|_p^2
    \quad \text{for all $\qbf, \check\qbf \in \simplex$.}
  \]
  For~$(i, j) \in \Nbb\times\Zbb_{\ge0}$, set
  \[
    t_{i,j} \defeq 2^{-i}j.
  \]
  Fix distinct~$\qbf, \check\qbf \in \simplex$ and define~$c: [0, 1] \to \simplex$ by
  \[
    c(t) \defeq (1-t)\qbf + t\check\qbf \quad \text{for $t \in [0,1]$.}
  \]
  We will show that 
  \begin{equation}
    \label{equation:induction_goal}
    f(c(t_{i,j})) \leq (1-t_{i,j}) f(c(0)) + t_{i,j}f(c(1))-\frac{\kappa^{f,\frac12}_p}{2}t_{i,j}(1-t_{i,j}) \|\qbf-\check\qbf\|_p^2
  \end{equation}
  for~$(i,j) \in \Nbb \times \Zbb_{\geq 0}$ with~$t_{i,j} \in [0,1]$ (namely, for~$i \in \Nbb$ and~$0 \leq j \leq 2^i$)
  by induction on~$i$.
  We immediately observe that \eqref{equation:induction_goal} always holds for~$t_{i,0}=0$ and~$t_{i,2^i}=1$ regardless of~$i$.
  
  The inequality~\eqref{equation:induction_goal} trivially holds for~$i=1$.
  Assume that \eqref{equation:induction_goal} holds for some~$i \in \Nbb$ and all~$j$ with~$0 \leq j \leq 2^i$.
  Then, \eqref{equation:induction_goal} also holds for~$t_{i+1,2j} = t_{i,j}$ with~$0 \leq j \leq 2^{i}$.
  For~$0 \leq j \leq 2^i-1$,
  we have
  \[
    t_{i+1,2j+1} = \frac{t_{i+1,2j}+t_{i+1,2j+2}}2 = \frac{t_{i,j}+t_{i,j+1}}2.
  \]
  Define~$c_{i,j}: [0,1] \to \simplex$ by
  \[
    c_{i,j}(t) \defeq c((1-t) \cdot t_{i,j} + t\cdot t_{i,j+1}) \quad \text{for $t \in [0,1]$.}
  \]
  This implies
  \begin{align*}
    f&(c(t_{i+1,2j+1})) \\
    &\leq \frac12 f(c_{i,j}(0)) + \frac12 f(c_{i,j}(1)) - \frac{\kappa^{f,\frac12}}{8}\|c_{i,j}(0)-c_{i,j}(1)\|_p^2\\
    &= \frac12 f(c_{i,j}(0)) + \frac12 f(c_{i,j}(1)) - \frac{\kappa^{f,\frac12}_p}{8} \cdot 2^{-2i} \cdot \|\qbf-\check\qbf\|_p^2\\
    &= \frac12 f(c(t_{i,j})) + \frac12 f(c(t_{i,j+1})) - \frac{\kappa^{f,\frac12}_p}{8} \cdot 2^{-2i} \cdot \|\qbf-\check\qbf\|_p^2\\
    &\leq \frac12\bigg[
      (1-t_{i,j}) f(c(0)) + t_{i,j}f(c(1)) - \frac{\kappa^{f,\frac12
      }_p}{2}t_{i,j}(1-t_{i,j}) \|\qbf-\check\qbf\|_p^2\bigg]\\
    &\phantom{\leq} + \frac12\bigg[
      (1-t_{i,j+1}) f(c(0)) + t_{i,j+1}f(c(1)) - \frac{\kappa^{f,\frac12}_p}{2}t_{i,j+1}(1-t_{i,j+1}) \|\qbf-\check\qbf\|_p^2\bigg]\\
    &\phantom{\leq} - \frac{\kappa^{f,\frac12}_p}{8} \cdot 2^{-2i} \cdot \|\qbf-\check\qbf\|_p^2\\
    &= (1-t_{i+1,2j+1}) f(c(0)) + t_{i+1,2j+1}f(c(1)) - \frac{\kappa^{f,\frac12}_p}{2}t_{i+1,2j+1}(1-t_{i+1,2j+1}) \|\qbf-\check\qbf\|_p^2
  \end{align*}
  as desired.
  
  Because~$f$ is continuous on~$\simplex$ and
  $\setcomp{t_{i,j} \in [0,1]}{(i, j) \in \Nbb \times \Zbb_{\geq 0}}$ is dense in~$[0,1]$,
  we find
  \[
    f((1-t)\qbf+t\check\qbf)\leq (1-t) f(\qbf) +tf(\check\qbf)-\frac{\kappa^{f,\frac12}_p}{2}t(1-t) \|\qbf-\check\qbf\|_p^2 \quad \text{for $t \in (0,1)$,}
  \] 
  which leads to
  \[
    \kappa^{f,\frac12}_p \leq
      \frac{
        2[(1-t) f(\qbf) +tf(\check\qbf)-
        f((1-t)\qbf+t\check\qbf)]}{t(1-t)\|\qbf-\check\qbf\|_p^2
      } \quad \text{for $t \in (0,1)$.}
  \]
  Since~$\qbf, \check\qbf \in \simplex$ are arbitrary, this ensures that~$\kappa^{f,t}_p \geq \kappa^{f,1/2}_p$ 
  for~$t\in (0,1)$ and completes the proof of the proposition.
\end{proof}

\moduluseval*

\begin{proof}
Assume that there exists $r_\ast\in (0,2^{1/p}]$ such that $K_p^f(r_\ast)=\kappa_p^f$.
We can see from \cref{lemma:supp:modulus_bound} that
\[
  \begin{aligned}
    \frac{K^{f}_p((1-2\tau)r_\ast)}{8}[(1-2\tau)r_\ast]^2
    &\leq \frac{K^f_p(r_\ast)}{8}r_\ast^2-\frac{\kappa^f_p}{2}\tau(1-\tau)r_\ast^2
    = \frac{K^f_p(r_\ast)}{8} [ (1-2\tau)r_\ast ]^2
  \end{aligned}
\]
for~$\tau \in (0,1/2)$.
Since~$K_p^f(r)\geq \kappa_p^f$ for~$r\in (0,2^{1/p}]$ (by \cref{definition:local_modulus} and \cref{equation:sc_modulus_local}), this implies~$K_p^f(r)= \kappa_p^f$ for $r\in (0,r_\ast]$ hence
\[
  \liminf_{r \downarrow 0} K^f_p(r) = \kappa_p^f.
\]
Assume that there is no~$r\in (0,2^{1/p}]$ so that~$K_p^f(r)=\kappa_p^f$.
Then, there exists~$(r_j)_{j \in \Nbb} \subset (0, 2^{1/p}]$ converging to~$0$ such that 
  \[
    \lim_{j\to\infty} K^f_p(r_j) = \inf\setcomp{K^{f}_{p}(r)}{r\in (0,2^{\frac1p}]}
    = \kappa^{f,\frac12}_p
    \leq \liminf_{r \downarrow 0} K^f_p(r)
    \leq \lim_{j \to \infty} K^f_p(r_j),
  \]
  where the second equality follows from \eqref{equation:sc_modulus_local}.
  This with \cref{proposition:sc_modulus_midpoint} proves the first assertion.
  
  For~$\qbf, \check\qbf\in \triangle$, we calculate 
  \[
    J(\qbf, \check\qbf)
    \geq \omega(\|\qbf - \check\qbf\|_p)
    = \frac{K^f_p(\|\qbf - \check\qbf\|_p)}{8}\|\qbf - \check\qbf\|_p^2
    \geq \frac{\kappa^f_p}{8}\|\qbf - \check\qbf\|_p^2.
  \]
  This is the second assertion.
  
  For~$r \in (0,2^{1/p}]$, 
  we can pick~$\qbf, \check\qbf \in \simplex$ such that~$\omega(r) = J(\qbf, \check\qbf)$ and~$\|\qbf - \check\qbf\|_p = r$ from \cref{lemma:moduli_minimizer}.
  For~$\tau \in (0,1/2)$, we have
  \[
    \omega((1-2\tau)r) \leq \omega(r) - \frac{\kappa^f_p}{2}\tau(1-\tau)r^2
  \]
  from \cref{lemma:supp:modulus_bound},
  which yields 
  \[
    D^-\omega(r)
    = \limsup_{\tau \downarrow 0} \frac{\omega(r)-\omega((1-2\tau)r)}{r-(1-2\tau)r}
    \geq \limsup_{\tau \downarrow 0} \frac{\dfrac{\kappa^f_p}{2}\tau(1-\tau)r^2}{2\tau r}
    = \frac{\kappa^f_p}{4}r.
  \]
  This completes the proof of the lemma.
\end{proof}

\section{Derivation of modulus for general \texorpdfstring{$N$}{N}}
\label{section:modulus_general_N}
In \cref{section:examples}, we mainly consider examples of~$\omega$ only for~$(N,p)=(2,1)$.
Since we have extended the moduli on general~$(\simplex,\|\cdot\|_p)$ in \cref{definition:modulus}, it is nice to have an example beyond the binary case.
To this end, we calculate~$\omega$ for the negative Shannon entropy~$f(\qbf) = \inpr{\qbf}{\ln\qbf}$ with general~$N \geq 2$ and~$p = 2$.
In what follows, we focus on~$f(\qbf) = \inpr{\qbf}{\ln\qbf}$, and the modulus~$\omega$ and midpoint Jensen gap~$J$ is defined based on this particular~$f$ throughout this section.

Let
\[
  \Ucal^{N-1} \defeq \setcomp{\ubf \in (0,1)^{N-1}}{\sum_{n\in[N-1]} u_n < 1}
\]
and define~$\psi: \Ucal^{N-1} \times \Ucal^{N-1} \to \Rbb$ and~$\phi: \Ucal^{N-1} \to \Rbb$ by
\[
  \begin{aligned}
    \psi(\ubf, \wbf) &\defeq \frac{1}{2}\Bigg\{ \sum_{n\in[N-1]}(u_n - w_n)^2 + \bigg[\sum_{n\in[N-1]}(u_n - w_n)\bigg]^2 \Bigg\} && \text{for $\ubf, \wbf \in \Ucal^{N-1}$,} \\
    \phi(\ubf) &\defeq \sum_{n\in[N-1]}u_n\ln u_n + \Bigg(1 - \sum_{n\in[N-1]}u_n\Bigg)\ln\Bigg(1 - \sum_{n\in[N-1]}u_n\Bigg) && \text{for $\ubf \in \Ucal^{N-1}$,}
  \end{aligned}
\]
respectively.
Here,~$\phi$ is the negative Shannon entropy of~$[\ubf\;\;1-\inpr{\ubf}{\onebf}]\in\simplex$.
Indeed, we have
\[
  \begin{bmatrix}
    u_1 \\ \vdots \\ u_{N-1} \\ 1 - \inpr{\ubf}{\onebf}
  \end{bmatrix}
  \in \simplex_+ \quad \text{and} \quad
  \psi(\ubf, \wbf) = \frac{1}{2}\left\|\begin{bmatrix}
    u_1 \\ \vdots \\ u_{N-1} \\ 1 - \inpr{\ubf}{\onebf}
  \end{bmatrix} - \begin{bmatrix}
    w_1 \\ \vdots \\ w_{N-1} \\ 1 - \inpr{\wbf}{\onebf}
  \end{bmatrix}\right\|_2^2,
\]
which yields~$\psi(\ubf, \wbf) \in [0,1]$.
First, we present a couple of necessary lemmas.

\begin{lemma}
  \label{lemma:sub2}
  For~$r \in (0, 2^{1/2})$, there exist~$\qbf, \check\qbf \in \simplex$ such that
  \[
    \|\qbf - \check\qbf\|_2 = r, \quad
    \supp(\qbf) \cap \supp(\check\qbf) \ne \emptyset, \quad
    J(\qbf, \check\qbf) < \ln2.
  \]
\end{lemma}
\begin{proof}
  Fix~$r\in (0,2^{1/2})$ and set~$a\defeq 2^{-1/2}r \in (0,1)$.
  Define~$\qbf, \check\qbf \in \simplex$ by
    \[
    q_n \defeq \delta_{1n}, \quad
    \check q_n \defeq (1-a)\delta_{1n} + a\delta_{2n}
    \quad \text{for $n \in [N]$.}
  \]
  Then, we have~$\|\qbf - \check\qbf\|_2 = r$,~$\supp(\qbf) \cap \supp(\check\qbf) \ne \emptyset$, and
 \begin{align*}
J(\qbf, \check\qbf)
=
\frac{1-a}{2} \ln (1-a)
-
\left(1-\frac{a}{2}\right)
\ln
\left(1-\frac{a}{2}\right)
+\frac{a}{2} \ln 2
\eqdef\overline{J}(a).
 \end{align*}
Since we have
\[
\overline{J}(1)=\ln2,\quad
\overline{J}'(a)=
\frac{1}{2} \ln\frac{2-a}{1-a}>0\quad \text{for }a<1,
\]
we conclude~$J(\qbf, \check\qbf)< \ln 2$ as desired.
\end{proof}

\begin{lemma}
  \label{lemma:sub3}
  For~$\qbf, \check\qbf \in \simplex$, if
  \[
    \supp(\qbf) \cap \supp(\check\qbf) \ne \emptyset, \quad
    \supp(\qbf) \ne \supp(\check\qbf),
  \]
  then there exist~$\qbf', \check\qbf' \in \simplex$ such that
  \[
    \|\qbf - \qbf\|_2 = \|\qbf' - \check\qbf'\|_2, \quad
    \supp(\qbf') = \supp(\check\qbf'), \quad
    J(\qbf', \check\qbf') < J(\qbf, \check\qbf).
  \]
\end{lemma}
\begin{proof}
  Take~$i \in \supp(\qbf) \cap \supp(\check\qbf)$ and write
  \[
    S_1 \defeq [\supp(\qbf) \cap \supp(\check\qbf)] \setminus \set{i}, \quad
    S_2 \defeq \supp(\qbf) \setminus \supp(\check\qbf), \quad
    S_3 \defeq \supp(\check\qbf) \setminus \supp(\qbf),
  \]
  and~$m \defeq |S_2 \cup S_3|$. %
  For sufficiently small~$\epsilon > 0$, define~$\qbf^\epsilon, \check\qbf^\epsilon \in \simplex$ by
  \[
    q_n^\epsilon \defeq \begin{cases}
      q_i - m\epsilon & \text{for $n = i$,} \\
      q_n & \text{for $n \in S_1$,} \\
      q_n + \epsilon & \text{for $n \in S_2 \cup S_3$,} \\
      0 & \text{otherwise,}
    \end{cases}
    \quad
    \check q_n^\epsilon \defeq \begin{cases}
      \check q_i - m\epsilon & \text{for $n = i$,} \\
      \check q_n & \text{for $n \in S_1$,} \\
      \check q_n + \epsilon & \text{for $n \in S_2 \cup S_3$,} \\
      0 & \text{otherwise.}
    \end{cases}
  \]
  We find that~$\|\qbf^\epsilon - \check\qbf^\epsilon\|_2 = \|\qbf - \check\qbf\|_2$,~$\supp(\qbf^\epsilon) = \supp(\check\qbf^\epsilon)$, and
  \[
    \lim_{\epsilon \downarrow 0}J(\qbf^\epsilon, \check\qbf^\epsilon) = J(\qbf, \check\qbf),
  \]
  thanks to the continuity of~$J$.
  Since we have
  \[
    \pdiff{}{\epsilon}J(\qbf^\epsilon, \check\qbf^\epsilon)
    = -\frac{m}{2}\ln\frac{(q_i-m\epsilon)(\check q_i-m\epsilon)}{\left(\frac{q_i+\check q_i}{2}-m\epsilon\right)^2} + \frac12\sum_{n \in S_2}\ln\frac{(q_n+\epsilon)\epsilon}{\left(\frac{q_n}{2}+\epsilon\right)^2} + \frac12\sum_{n \in S_3}\ln\frac{(\check q_n+\epsilon)\epsilon}{\left(\frac{\check q_n}{2}+\epsilon\right)^2},
  \]
  which diverges to~$-\infty$ as~$\epsilon \downarrow 0$,
  we have~$J(\qbf^\epsilon, \check\qbf^\epsilon) < J(\qbf, \check\qbf)$ for sufficiently small~$\epsilon > 0$.
  This completes the proof of the lemma.
\end{proof}
\begin{corollary}
   \label{lemma:support_match}
   For~$\qbf, \check\qbf \in \simplex$ with $\|\qbf - \check\qbf\|_2 < 2^{1/2}$, if $\omega(\|\qbf - \check\qbf\|_2) = J(\qbf, \check\qbf)$ holds then~$\supp(\qbf) = \supp(\check\qbf)$.
 \end{corollary} 
\begin{proof}
Let~$\qbf, \check\qbf \in \simplex$ satisfy~$\|\qbf - \check\qbf\|_2 < 2^{1/2}$.
By \cref{lemma:sub2},
we have~$\omega(\|\qbf-\check\qbf\|_2)<\ln 2$.
It is easy to see  if  $\supp(\qbf')\cap \supp(\check\qbf')=\emptyset$, then $J(\qbf', \check\qbf') = \ln2$ holds.
This with  \cref{lemma:sub3} leads to ~$\omega(\|\qbf - \check\qbf\|_2) = J(\qbf, \check\qbf)$
implies~$\supp(\qbf) = \supp(\check\qbf)$.
Thus, the proof of the corollary is complete.
\end{proof} 

Subsequently, we show \cref{lemma:rank_zero,lemma:lagrangian}, which are needed to invoke the method of Lagrangian multipliers later.

\begin{lemma}
  \label{lemma:rank_zero}
  For~$\ubf, \wbf \in \Ucal^{N-1}$, the rank of the Jacobian of~$\psi$ at~$(\ubf, \wbf)$ is zero if and only if $\ubf=\wbf$.
\end{lemma}
\begin{proof}%
  Since we have
  \[
    \pdiff{\psi}{u_i}(\ubf, \wbf) = u_i - w_i + \sum_{n\in[N-1]}(u_n - w_n)
    = -\pdiff{\psi}{w_i}
    \quad \text{for $i \in [N-1]$,}
  \]
  the rank of the Jacobian of~$\psi$ at~$(\ubf, \wbf)$ is zero if and only if
  \[
    u_i - w_i + \sum_{n\in[N-1]}(u_n - w_n) = 0
    \quad \text{for $i \in [N-1]$.}
  \]
  Summing the above equation up gives
  \[
    N \sum_{n\in[N-1]}(u_n - w_n) = 0,
  \]
  and hence~$u_i - w_i = 0$ for all~$i \in [N-1]$,
  which shows $\ubf=\wbf$. %
  The converse implication is trivial.
\end{proof}

\begin{lemma}
  \label{lemma:lagrangian}
For~$r \in (0,2^{1/2})$ with $r^2<N/(N-1)$, let~$\ubf, \wbf \in \Ucal^{N-1}$ satisfy~$\psi(\ubf,\wbf) = r^2/2$
and set
\[
  u_N \defeq 1 - \sum_{n\in[N-1]}u_n, \quad
    w_N \defeq 1 -\sum_{n\in[N-1]}w_n.
\]
If there exists~$\lambda \in \Rbb$ such that
  \begin{equation}
    \label{equation:supp:lagrangian}
    \begin{aligned}
      \frac{1}{2}\pdiff{\phi}{u_n}(\ubf) - \pdiff{\phi}{u_n}\left(\frac{\ubf + \wbf}{2}\right) + \lambda\pdiff{\psi}{u_n}(\ubf, \wbf) &= 0 \quad \text{and} \\
      \frac{1}{2}\pdiff{\phi}{w_n}(\ubf) - \pdiff{\phi}{w_n}\left(\frac{\ubf + \wbf}{2}\right) + \lambda\pdiff{\psi}{w_n}(\ubf, \wbf) &= 0
    \end{aligned}
  \end{equation}
  for~$n \in [N-1]$,
then there exists $\Ical \subset [N]$ 
with $I\defeq |\Ical|$ such that $0<I(N-I)<r^{-2}N$ 
and
\begin{align*}
    u_n& = \frac{1+r\sqrt{a_{N}(I)}}{2I}, \quad
    w_n=\frac{1-r\sqrt{a_{N}(I)}}{2I}\quad \text{for $n \in \Ical$,}\\
    u_n &=  \frac{1-r\sqrt{a_{N}(I)}}{2(N-I)},\quad
    w_n = \frac{1+r\sqrt{a_{N}(I)}}{2(N-I)} \quad \text{for $n \in [N]\setminus \Ical$},
    \quad\text{where }
a_{N}(I)\coloneqq \frac{I(N-I)}{N}.    
\end{align*}
\end{lemma}
\begin{proof}%
  For simplicity, set
  \[
    \vbf \defeq \frac{\ubf + \wbf}{2}, \quad
  \text{and} \quad
    v_N \defeq 1 - \sum_{n\in[N-1]}v_n.
  \]
  Assuming \cref{equation:supp:lagrangian}, we calculate
  \[
    \begin{aligned}
      0
      &= \frac12 \pdiff{\phi}{u_n}(\ubf) - \pdiff{\phi}{u_n}(\vbf) + \lambda\pdiff{\psi}{u_n}(\ubf, \wbf)
      = \frac12\left(\ln\frac{u_n}{u_N} - \ln\frac{v_n}{v_N}\right) + \lambda[u_n - w_n - (u_N - w_N)], \\
      0
      &= \frac12 \pdiff{\phi}{w_n}(\wbf) - \pdiff{\phi}{w_n}(\vbf) + \lambda\pdiff{\psi}{w_n}(\ubf, \wbf)
      = \frac12\left(\ln\frac{w_n}{w_N} - \ln\frac{v_n}{v_N}\right) - \lambda[u_n - w_n - (u_N - w_N)],
    \end{aligned}
  \]
  for~$n \in [N-1]$, which yields
  \begin{equation}\label{eq:1}
    \begin{aligned}
      \ln\frac{u_n}{u_N} - \ln\frac{v_n}{v_N}
      = -2\lambda[u_n - w_n - (u_N - w_N)]
      = -\ln\frac{w_n}{w_N} + \ln\frac{v_n}{v_N}
    \end{aligned}
  \end{equation}
  Thus, we have
  \[
\frac{(u_n+w_n)^2}{4u_nw_n}=\frac{v_n^2}{u_nw_n} = \frac{v_N^2}{u_N w_N}
\eqdef \Lambda>1\quad
\text{for $n \in [N-1]$},
  \]
where~$\Lambda>1$ holds; otherwise~$u_n=w_n$ holds for all~$n\in[N]$, which contradicts~$\psi(\ubf,\wbf)=r^2/2$.
By rearranging in~$w_n$, we have
\[
w_n^2- 2u_n(2\Lambda- 1)w_n+u_n^2=0\quad
\text{for $n \in [N]$.}
\]
Setting
\[
\mu\defeq (2\Lambda- 1)-2\sqrt{\Lambda^2- \Lambda}\in (0,1),
\]
we have either $w_n=\mu u_n$ or $w_n=\mu^{-1} u_n$
for $n\in [N]$, by solving the quadratic equation with respect to~$w_n$.
Set
\[
\Ical\coloneqq\{i\in [N] \mid w_i=\mu u_i \},\quad
I\defeq |\Ical|,\qquad
\Jcal\coloneqq\{j\in [N] \mid w_j=\mu^{-1} u_j\},\quad
J\defeq|\Jcal|.
\]
Then $I+J=N$ holds.

Assume $N\in \Ical$.
If $\Jcal=\emptyset$,
then 
\begin{align*}
w_N
=1-\sum_{i\in[N-1]} w_i
&=1-\mu \sum_{i\in[N-1]}u_i\\
&\neq 
\mu-\mu \sum_{i\in[N-1]}u_i
=\mu \left(1-\sum_{i\in[N-1]}u_i\right)
=\mu u_N=w_N,
\end{align*}
which is a contradiction.
Thus $\Jcal\neq \emptyset$ holds.
On one hand, for $j\in \Jcal$, 
we observe from \cref{eq:1}  that 
\[
\ln\frac{1+\mu}{1+\mu^{-1}}
=\ln\frac{u_j}{u_N} - \ln\frac{v_j}{v_N}
= -2\lambda[u_j - w_j - (u_N - w_N)]\\
=2\lambda(1-\mu)(\mu^{-1}u_j+u_N).
\]
Since the left-hand side is independent of $j$ and not zero,
then so is the right-hand side hence $u_j$ is determined independent of $j$ and $\lambda \neq 0$.
On the other hand, for $i\in \Ical$, it turns out that 
\[
\ln \frac{u_i}{u_N}-\ln\frac{v_i}{v_N}=0
\]
and consequently $u_i=u_N$ holds
by \cref{eq:1} together with the property $\lambda \neq 0$.
This yields $u_i=u_N$ for $i\in \Ical$.
Thus there exist $s,t\in (0,1)$ such that 
\begin{equation}\label{2}
u_i=s, \ w_i=\mu s \quad \text{for }i\in \Ical, 
\quad
u_j=t, \ w_j=\mu^{-1} t \quad \text{for }j\in \Jcal.
\end{equation}
We see that 
\[
1=\sum_{n\in [N]} u_N=Is+Jt,\quad
1=\sum_{n\in [N]} w_N=\mu Is+\mu^{-1}Jt,
\]
that is,~$Jt=1-Is=\mu(1-\mu Is)$.
This can be simplified as follows:
\[
\begin{aligned}
  \mu &= \frac{1}{Is}-1 = \frac{Jt}{Is}.
\end{aligned}
\]
We also find that 
\begin{align*}
r^2
&=2\psi(\ubf,\wbf)
= (1-\mu)^2
\left(Is^2+ \mu^{-2}Jt^2\right)
= (1-\mu)^2
\left(Is^2+ \frac{I^2}{J} s^2\right)
= \frac{I}{J}N(1-\mu)^2 s^2,
\end{align*}
which in turn shows
\[
r\sqrt{\frac{J}{NI}}=(1-\mu) s=2s-\frac{1}{I}.
\]
We conclude 
\[
s=\frac{1}{2I}\left(1+r\sqrt{\frac{IJ}{N}}\right),\quad
t=\frac{1}{2J}\left(1-r\sqrt{\frac{IJ}{N}}\right)
\]
as desired.

The case $N\in \Jcal$ is proved by switching the role of $\ubf$ and $\wbf$ in the argument for the case $N\in \Ical$, and the proof is achieved.
\end{proof}

By combining these lemmas, we have the following claim, which is the minimizers~$(\qbf, \check\qbf)$ for the negative Shannon entropy we show in \cref{section:examples}.
\begin{proposition}
  \label{proposition:shannon_general_minimizer}
  For~$r \in (0,2^{1/2})$ with $r^2<N/(N-1)$,~$\qbf, \check\qbf \in \simplex$ satisfy~$\|\qbf - \check\qbf\|_2 = r$ and~$\omega(r) = J(\qbf, \check\qbf)$
  if and only if there exist distinct~$i, j \in [N]$ such that
  \[
    q_n \defeq \frac{1+2^{-1/2}r}{2}\delta_{ni} + \frac{1-2^{-1/2}r}{2}\delta_{nj}, \quad
    \check q_n \defeq \frac{1-2^{-1/2}r}{2}\delta_{ni} + \frac{1+2^{-1/2}r}{2}\delta_{nj}.
  \]
In this case,
\[
\omega(r)=J(\qbf, \check\qbf)
=\frac12
\left[
\left(1+\frac{r}{\sqrt{2}}\right)
\ln 
\left(1+\frac{r}{\sqrt{2}}\right)
+
\left(1-\frac{r}{\sqrt{2}}\right)
\ln\left(1-\frac{r}{\sqrt{2}}\right)
\right].
\]
\end{proposition}
\begin{proof}
  There exist~$\qbf', \check\qbf' \in \simplex$ such that
  \[
    \|\qbf' - \check\qbf'\|_2 = r, \quad \omega(r) = J(\qbf, \check\qbf),
    \quad \text{and} \quad \supp(\qbf') = \supp(\check\qbf')
  \]
  from \cref{lemma:moduli_minimizer} and \cref{lemma:support_match}.
By relabeling the indices~$n \in [N]$ and switching~$\qbf'$ and~$\check\qbf'$ 
if necessary, we may assume that
and there exists~$N' \in [N]$ %
  such that $N' \geq 2$ with 
\[
\supp(\qbf') = \supp(\check\qbf') = [N'],
\]
and $I\in [N'-1]$ with $I\leq N'/2$ such that $0<I(N'-I)<r^{-2}N'$ 
with
\begin{align}\label{eq:3}
\begin{split}
q'_n& = \frac{1+r\sqrt{a_{N'}(I)}}{2I}, \quad
\check{q}'_n=\frac{1-r\sqrt{a_{N'}(I)}}{2I}
\quad \text{for $n \in [I]$,}\\
q'_n &=  \frac{1-r\sqrt{a_{N'}(I)}}{2(N'-I)},\quad
\check{q}'_n = \frac{1+r\sqrt{a_{N'}(I)}}{2(N'-I)}    \quad \text{for $n \in [N']\setminus [I]$},
\end{split}
\end{align}
by \cref{lemma:lagrangian},
where 
\[
a_{N'}(x)\defeq \frac{x(N'-x)}{N'}.
\]
We need to identify~$I$ and~$N'$ hereafter.
Note that for any $N'\in [N]$, we have
\[
1\cdot\left(1-\frac{1}{N'}\right)
\leq 1\cdot\left(1-\frac{1}{N}\right)
\leq r^{-2}
\]
hence 
\[
I'\defeq \max\setcomp{i\in [N'-1]}{i(N'-i)<r^{-2} N', i\leq N'/2}
\]
is well-defined.

Setting 
\[
\overline{J}_{N'}(x)
\defeq
\left[1+r\sqrt{a_{N'}(x)}\right] \ln\left[1+r\sqrt{a_{N'}(x)}\right]
+
\left[1-r\sqrt{a_{N'}(x)}\right] \ln\left[1-r\sqrt{a_{N'}(x)}\right]
\]
for $1\leq x \leq N'-1$,
we see that 
\begin{align*}
J(\qbf',\check\qbf')
=\frac12\overline{J}_{N'}(I).
\end{align*}
From this with the relation
\[
\begin{aligned}
\frac{\rd}{\rd{x}}\overline{J}_{N'}(x)
&= \frac{r}{2\sqrt{a_{N'}(x)}}\left(1-\frac{2x}{N'}\right)
\ln \frac{1+r\sqrt{a_{N'}(x)}}{1-r\sqrt{a_{N'}(x)}} 
\ge 0 && \text{for $x\in[I']$,}
\end{aligned}
\]
we observe
\[
\min_{I\in [I']}
\overline{J}_{N'}(I)
=\overline{J}_{N'}(1)
= \overline{J}(N')
\]
where
\[
\overline{J}(y)
\defeq
\left(1+r\sqrt{1-y^{-1}}\right)
\ln 
\left(1+r\sqrt{1-y^{-1}}\right)
+
\left(1-r\sqrt{1-y^{-1}}\right)
\ln
\left(1-r\sqrt{1-y^{-1}}\right)
\]
for $y\geq 2$.
We calculate
\begin{align*}
\frac{\rd}{\rd{y}}\overline{J}(y)
=\frac{ry^{-2}}{2\sqrt{1-y^{-1}}}
\ln\frac{1+r\sqrt{1-y^{-1}}}{1-r\sqrt{1-y^{-1}}}
>0
\quad\text{for }y\geq 2,
\end{align*}
which leads to 
\[
\min_{N'\in [N], N'\geq 2}\overline{J}(N')
=\overline{J}(2).
\]
Thus, in \cref{eq:3}, 
the correct choice is $(N',I)=(2,1)$ and
this completes the proof of the proposition.
\end{proof}

\end{document}